\documentclass[11pt]{article}
\usepackage{geometry} 
\usepackage{microtype}
\usepackage{graphicx}
\usepackage{subfigure}
\usepackage{booktabs}
\usepackage{graphicx,psfrag,amsmath,amsthm,amsfonts,verbatim,comment,mathtools}
\usepackage{amssymb}
\usepackage[mathscr]{euscript}
\usepackage{hyperref}

\newtheorem{theorem}{Theorem}[section]
\newtheorem{proposition}[theorem]{Proposition}
\newtheorem{corollary}[theorem]{Corollary}
\newtheorem{lemma}[theorem]{Lemma}

\newtheorem{example}{Example}[section]
\newcounter{rmnum}
\newenvironment{romannum}{\begin{list}{{\upshape (\roman{rmnum})}}{\usecounter{rmnum}
			\setlength{\leftmargin}{12pt}
			\setlength{\rightmargin}{10pt}
			\setlength{\itemindent}{-1pt}
	}}{\end{list}}

\def\E{\mathbb{E}}

\def\Pr{\mathbb{P}}

\def\1{\mathbf{1}}

\def\N{\mathbb{N}}

\def\epsy{\varepsilon}

\def\bfmR{\boldsymbol{R}}

\def\Expect{{\mathbb E}}

\newcommand{\field}[1]{\mathbb{#1}}
\def\ind{\field{I}}

\def\clA{{\mathcal A}}
\def\clF{{\mathcal F}}
\def\clP{{\mathcal P}}

\def\barGamma{{\overline{ {\Gamma} }}}

\def\bfmalpha{\boldsymbol{\alpha}}

\def\bfmdelta{\boldsymbol{\delta}}

\def\eqdef{\mathbin{:=}}

\DeclareMathOperator*{\argmax}{arg\,max}
\DeclareMathOperator*{\argmin}{arg\,min}
\hypersetup{
	colorlinks=true,
	linkcolor=blue,
	filecolor=magenta,      
	urlcolor=purple,
}
\usepackage{xcolor}

\newcount\Comments
\newcommand{\kibitz}[2]{\ifnum\Comments=1{\textcolor{#1}{\textsf{\footnotesize #2}}}\fi}
\usepackage{color}
\definecolor{darkred}{rgb}{0.7,0,0}
\definecolor{darkgreen}{rgb}{0.0,0.5,0.0}
\definecolor{darkblue}{rgb}{0.0,0.0,0.5}
\definecolor{teal}{rgb}{0.0,0.5,0.5}

\graphicspath{Figures/}

\begin{document}

\title{A Bit Better? \\ Quantifying Information for Bandit Learning}
\author{Adithya M. Devraj
\\ {adevraj@stanford.edu} 
\and Benjamin Van Roy
\\ {bvr@stanford.edu}
\and Kuang Xu
\\ {kuangxu@stanford.edu}
}
\date{
\large 
Stanford University, Stanford, CA 94305
\\[1em]
\today
}

\maketitle

\begin{abstract}
\normalsize
The information ratio offers an approach to assessing the efficacy with which an agent balances between exploration and exploitation. Originally, this was defined to be the ratio between squared expected regret and the mutual information between the environment and action-observation pair, which represents a measure of information gain. Recent work has inspired consideration of alternative information measures, particularly for use in analysis of bandit learning algorithms to arrive at tighter regret bounds.  We investigate whether quantification of information via such alternatives can improve the realized performance of information-directed sampling, which aims to minimize the information ratio.
\\[1em]
{\bf Acknowledgement:} Financial support from Army Research Office (ARO) grant W911NF2010055 is gratefully acknowledged.
\end{abstract}

\section{Introduction}
\label{s:intro}

We consider the multi-arm bandit problem with independent arms.  At each time, an agent executes an action and observes a reward.  The objective is to maximize expected cumulative reward over a long time horizon.  This problem crystallizes the \emph{exploration-exploitation} dilemma.

The agent's choice of action must balance between expected immediate reward and information that may increase subsequent rewards.  This calls for quantification of the information gain, or equivalently, the reduction in uncertainty, that results from observing the outcome of an action.  Taking cue from information theory, it is natural to consider \emph{Shannon entropy}.  This forms the basis of the \emph{information ratio}, as originally introduced in \cite{rusvan14}.  This statistic is the ratio between squared expected instantaneous regret and the mutual information between the multi-armed bandit and the action-reward pair, or equivalently, the reduction in entropy resulting from observing the action and reward.  It has been used in the analysis and design of several bandit algorithms \cite{rusvan14,rusvan16,rusvan18l,rusvan18s,latsze19,latgyo20,donvan18,luvan19}.

Information-directed sampling (IDS), as originally introduced by \cite{rusvan14}, selects at each time a randomized action that minimizes the information ratio.  The algorithm was shown to achieve, for each $K$ and $T$, worst case Bayesian regret in excess of optimal by a factor of at most $\sqrt{\log(K)}$, where $T$ is the horizon and $K$ is the number of arms. 
Indeed, Shannon entropy measures bits acquired by the agent, and while this notion is fundamental to the field of communication, it is not immediately clear whether this is the right notion for balancing exploration and exploitation in bandit learning.

In recent work \cite{latsze19}, it was shown that the IDS algorithm with information ratio defined using another information measure known as \emph{Tsallis entropy} leads to a better regret bound, that is optimal up to a constant factor. 

This paper contributes to understanding how to best quantify information for bandit learning. Specifically, we investigate whether the performance of IDS with Tsallis entropy exceeds that with Shannon entropy, as suggested by the aforementioned regret bounds.  Alternatively, the apparent advantage could just be a figment of current analytic techniques.  Our findings are as follows:
\begin{romannum}
\item 

We show that it is impossible to improve the aforementioned regret bound for IDS with Shannon entropy, using the current template for analysis based on the information ratio.

\item 

We propose a modification to this template, which we use to obtain order optimal bounds for Thompson sampling using Shannon entropy definition of information ratio, which was not possible using the previous template.

\item 

We present a computational study which suggests that, despite the gap in regret bounds, the realized performance of IDS with Shannon entropy approximately matches that of IDS with Tsallis entropy.

\end{romannum}

\subsection*{Literature Review}

One of the first finite-time regret bounds for Thompson sampling was derived in \cite{agrgoy13}. Building on the techniques of \cite{audbub09,rusvan14l}, it was first shown in \cite{bubliu14} that Thompson sampling achieves an optimal regret bound, up to a constant factor.

The concepts of information ratio and IDS was first introduced in \cite{rusvan14,rusvan18l}. A frequentist version of the algorithm was later developed in \cite{kirkra18}, and was extended to a linear partial monitoring setting in \cite{kirlatkra20}. An asymptotically optimal version of IDS was recently introduced in \cite{kirlatversze20}.

It was first shown in \cite{rusvan16} that an information ratio based analysis can be used to obtain regret bounds for Thompson sampling. Identical analysis, but with a different information measure was shown to produce a tighter regret bound for Thompson sampling in \cite{latsze19}. Information ratio based analysis has also been useful in adversarial settings \cite{bubdekkorper15,zimlat19,lat20,latsze20e}.
\nocite{latsze20b,bubaud12}

\section{Problem Formulation}
\label{s:setup}

Consider the probability space $(\Omega, \clF, \Pr)$. All random variables under consideration are defined with respect to this probability space.

For $K \geq 2$, let $\mathcal{P}$ be the set of probability measures on $[0,1]^K$, and let $P_*$ be a (deterministic) probability distribution over $\mathcal{P}$.  Denote $p^*$ to be a random variable that takes values in $\clP$ such that $\Pr(p_* \in \cdot ) = P_*(\cdot)$.

Let $\clA = \{1,\ldots,K\}$ denote the action set, and let the optimal action $A_*$ be a random element of $\mathcal{A}$ that satisfies 
\[
A_* \in \argmax_{a \in \clA} \int_{[0,1]^K} r_a p_*(dr) 
\]
with $r_a$ denoting the $a^{\text{th}}$ 
component of $r \in [0,1]^K$.
Let $\bfmR \eqdef \{R_t:t=1,\ldots,T\}$ be a random sequence that is i.i.d. conditioned on $p_*$, with each element distributed according to $p_*$.  In other words, for all $s \neq t$, $R_s$ is independent of $R_t$ conditioned on $p_*$, and
\[
\Pr(R_s \in \cdot | p_*) = \Pr(R_t \in \cdot | p_*) = p_*(\cdot)
\]  
We also assume that the reward sequence $\bfmR$ is independent of $A_*$, conditioned on $p_*$, and define 
\[
R_* = \max_{a \in \clA} \int_{[0,1]^K} r_a p_*(dr)
\]

\subsection{Policy and Bayesian Regret}

At each time-step $t \geq 0$, an agent selects an action $A_t$ based on the history of observations $H_t$, and receives a reward $R_{t+1,A_t}$, where
\[
H_t \! = \! \{A_s, R_{s+1, A_s}:s = 0,1,2,\ldots,t-1\}\,, \,\, 0 \leq t \leq T-1
\]
and $R_{s, a}$ denotes the $a^{\text{th}}$ component of $R_{s} \in [0,1]^K$.
Formally, a policy $\pi$ is a deterministic function, 
where, for each $t \geq 1$, $\pi(H_t)$ specifies a probability distribution over the action set $\clA$. 
With abuse of notation we will denote this distribution as $\pi$, where $\pi(a) = \Pr(A_t = a | H_t)$ denotes the probability with which the agent chooses action $a$, given the observed history $H_t$. Note that $H_0 = \emptyset$.

For the Thompson sampling agent, the action sequence $\{A_t:t=0,1,\ldots,T-1\}$ satisfies \cite{tho33,rusvankazosbwen17},
\begin{equation}
\pi^{\mathrm{TS}}(a) = \Pr(A_t = a | H_t) = \Pr(A_* = a | H_t)
\label{e:TS_prob_matching}
\end{equation}

The instance regret associated with a policy $\pi$ is defined to be the regret of the agent conditioned on $p_*$:
\begin{equation*}
\mathrm{Regret}(T, p_*,\pi) = \E\left[\sum_{t=0}^{T-1} (R_* - R_{t+1, A_t}) \Big| p_*\right]
\label{e:regret_instance}
\end{equation*}
The objective of interest in this work is the Bayesian regret, that is an expectation over the randomness of $p_*$:
\begin{equation}
\mathrm{Regret}(T, \pi) = \E\left[\sum_{t=0}^{T-1} (R_* - R_{t+1, A_t})\right]
\label{e:regret_bayes}
\end{equation}
where we have overloaded notation with an understanding that the definition of $\mathrm{Regret}$ depends on its arguments.

\subsection{Notations and Definitions}

We will denote by short hand  $\Pr_t(\cdot) = \Pr(\cdot | H_t ) = \Pr(\cdot | A_0, R_{1, A_0}, \ldots, A_{t-1}, R_{t, A_{t-1}})$, and $\Expect_t[\cdot] = \Expect[\cdot | H_t]$. 

The relative entropy $D_{\mathrm{KL}} (u || v)$ between probability measures $u$ and $v$ on the same measurable space is
\begin{equation}
D_{\mathrm{KL}} (u || v) =
\begin{cases}
\int \log \Big( \frac{du}{dv}\Big) \, du \qquad 
& \text{if} \quad u \ll v
\\
\infty \qquad & \text{otherwise}
\end{cases}
\label{e:KL}
\end{equation}

For a convex function $F: \Re^K \to \Re \cup \{\infty\}$, we denote by $\mathrm{dom}(F) = \{u: F(u) < \infty\}$ the domain of $F$. For a convex or differentiable $F$, the Bregman divergence $D_F:  \mathrm{dom}(F) \times \mathrm{dom}(F) \to [0,\infty]$ is defined as
\begin{equation}
D_F(u, v) \eqdef F(u) - F(v) - \nabla_{u-v} F(v) \,, \,\, u,v \in \Re^K
\label{e:BregmanDiv}
\end{equation}
where $\nabla_{u-v} F(v)$ denotes the directional derivative of $F$ in the direction $u-v$, at $v$. Throughout, we will refer to $F$ as the \emph{potential function}. 

We will denote $\Delta^{K-1}$ to be the $(K\!-\!1)$-dimensional probability simplex: $\Delta^{K-1} = \{ u \in [0,1]^K : \| u \|_1 = 1\}$, and
\begin{equation}
{{ \mathrm{diam}_F(\Delta^{K-1}) \eqdef \sup_{u,v \in \Delta^{K-1}} F(u) - F(v)}}
\label{e:diam_F}
\end{equation}

The relative entropy $D_{\mathrm{KL}} (u || v)$ between categorical distributions $u,v \in \Delta^{K-1}$ is the
Bregman divergence $D_F(u,v)$ associated with the unnormalized negentropy potential
\begin{equation}
F(u) = \sum_{i\in \clA: u_i > 0} (u_i \log(u_i) - u_i)
\label{e:negentropy}
\end{equation}
Letting $F$ be the $1/2$-Tsallis entropy, we have
\begin{align}
F(u) & = - 2 \sum_{i \in \clA} \sqrt{u_i}
\label{e:halftsallis}
\\
D_F(u,v) & = \sum_{i \in \clA: v_i > 0} \frac{(\sqrt{u_i} - \sqrt{v_i})^2}{\sqrt{v_i}}
\nonumber
\end{align}

\subsection{Information Ratio and Information Directed Sampling}

Given $P_*$, and a potential function $F$, denote $\Delta_t(a)$ to be the expected instantaneous regret of taking action $a \in \clA$ at time $t$: 
$$\Delta_t(a) \eqdef  \Expect_t \big[ R_{*} - R_{t+1,a}   \big] $$
and $g_t^F(a)$ to be the expected reduction in entropy,
$$g_t^F(a) \eqdef \Expect_t  \big[ D_F \big( \Pr_{t+1}(A_* \in \cdot ), \Pr_t(A_* \in \cdot) \big ) \mid A_t = a \big]$$
For a policy $\pi$, we overload the notation for $\Delta_t(\cdot)$ and $g_t^F(\cdot)$ and denote
\[
\begin{aligned}
\Delta_t(\pi) & = \sum_{a \in \clA} \pi(a) \Delta_t(a)
\\
g_t^F(\pi) & = \sum_{a \in \clA} \pi(a) g_t^F(a)
\end{aligned}
\]

For any potential $F$ that satisfies $\mathrm{diam}_F(\Delta^{K-1}) < \infty$, the information ratio $\Gamma_t^F (\pi)$ associated with policy $\pi$ at time-step $t$ is
\begin{equation}
\Gamma_t^F (\pi) 
\eqdef \frac{\big[ \Delta_t(\pi)  \big]^2}{g_t^F(\pi)}
\label{e:IR}
\end{equation}
The IDS algorithm greedily minimizes the information ratio at each time-step \cite{rusvan14,rusvan18l}:
\begin{equation}
\pi^{\mathrm{IDS}} \in \argmin_{\pi} \,\, \Gamma_t^F (\pi) 
\label{e:IDS}
\end{equation}
The original IDS algorithm in \cite{rusvan14,rusvan18l} considered the special case of $F$ being the negentropy potential. We will refer to the resulting algorithm as Shannon-IDS or
$\mathrm{NDS}$. Extension to $F$ being $1/2$-Tsallis entropy was considered in \cite{latsze19}. We refer to the resulting algorithm as Tsallis-IDS or $\mathrm{TDS}$.

\section{Information Ratio and Bayesian Regret}
\label{s:IR_BR}

Here, we review a general information ratio based analysis technique that can be used to obtain upper bounds for the Bayesian regret of any policy. 

Given $T \in \N$, and  $F: \Re^{K} \to \Re \cup \{\infty\}$, denote $\barGamma^F_T(\pi)$ to be the average expected information ratio corresponding to policy $\pi$:
\begin{align}
& \barGamma^F_T(\pi) \eqdef \frac{1}{T} \,\, \Expect \left [ \sum_{t=0}^{T-1}  \Gamma_t^F (\pi) \right]
\label{e:IR_avg}
\end{align}

Theorem~\ref{t:Rprog} provides an upper bound for the regret of any policy $\pi$ in terms of $\barGamma^F_T(\pi)$.
\begin{theorem}
\label{t:Rprog}
For any policy $\pi$, $T \in \N$, and $K\geq2$,
\begin{equation}
{\mathrm{Regret}}(T, \pi) \leq \sqrt{ \barGamma^F_T(\pi) \cdot {\mathrm{diam}_F(\Delta^{K-1})} \cdot T  }
\label{e:Rprog}
\end{equation}
where $F$ is convex, and satisfies $\mathrm{diam}_F(\Delta^{K-1}) < \infty$.
\end{theorem}
The proof is very similar to the proof of Theorem~3 in \cite{latsze19} and is provided in Appendix~\ref{s:RprogExt_proof}. Similar versions of the result previously appeared in  \cite{rusvan16,rusvan18l}. 

For any potential function $F$, and policy $\pi$, let $\gamma_F (\pi)$ be a uniform upper bound on the information ratio:
$$\Gamma_t^F (\pi)  \leq \gamma_F (\pi) \,\, a.s., \quad t = 0,1,\ldots$$
The following result is an immediate corollary to Theorem~\ref{t:Rprog}.
\begin{corollary}
\label{t:Rprog_uniform_bd}
For any policy $\pi$, $T \in \N$, and $K\geq2$,
\begin{equation}
{\mathrm{Regret}}(T, \pi) \leq \sqrt{ \gamma_F (\pi) \cdot {\mathrm{diam}_F(\Delta^{K-1})} \cdot T }
\label{e:Rprog_uniform_bd}
\end{equation}
where $F$ is convex, and satisfies $\mathrm{diam}_F(\Delta^{K-1}) < \infty$.
\end{corollary}

Table~\ref{tab:pot_gamma} summarizes the best known upper bounds for the information ratio and $ \mathrm{diam}_F(\Delta^{K-1})$ for the two potential functions of interest.
\begin{table}[ht]
\begin{center}
\small{
\begin{tabular}{|c|c|c|c|c|}
\hline
&  &  &  & \\[-0.4em]
$F$ &  $\gamma_F (\pi^{\mathrm{TS}})$   &   $\gamma_F (\pi^{\mathrm{NDS}})$ &    $\gamma_F (\pi^{\mathrm{TDS}})$   & $ \mathrm{diam}_F(\Delta^{K-1})$  \\ [0.3em]
\hline
&  &  &  & \\[-0.4em]
\hspace{-0.1in}  $1/2$-Tsallis \hspace{-0.1in}  & $\sqrt{K} $ & - & $\sqrt{K} $ &  $2\sqrt{K}$ \\[0.3em]  
Negentropy & ${K}/{2}$ & ${K}/{2}$ & - & $\log(K)$  \\ [0.5em]
\hline
\end{tabular}
}
\caption{Different potential functions $F$ and the corresponding best known upper bounds on the information ratio and $\mathrm{diam}_F(\Delta^{K-1})$.}
\end{center}
\label{tab:pot_gamma}
\vspace{-0.2in}
\end{table}

The following Proposition~\ref{t:Rprog_application} is a direct consequence of Corollary~\ref{t:Rprog_uniform_bd} and the bounds in Table~\ref{tab:pot_gamma}.

\begin{proposition}
\label{t:Rprog_application}
For all $T \in \N$ and $K \geq 2$,
\begin{subequations}
\begin{align}
{\mathrm{Regret}}(T, \pi^{\mathrm{TS}}) & \leq \sqrt{2 K T} 
\label{e:TS_reg_bd}
\\[1em]
{\mathrm{Regret}}(T, \pi^{\mathrm{TDS}}) & \leq \sqrt{2 K T}
\label{e:HIDS_reg_bd}
\\[1em]
{\mathrm{Regret}}(T, \pi^{\mathrm{NDS}}) & \leq \sqrt{\frac{1}{2} K T \log(K)}
\label{e:MIDS_reg_bd}
\end{align}
\end{subequations}
\end{proposition}
The proof of \eqref{e:TS_reg_bd} and \eqref{e:HIDS_reg_bd} of Proposition~\ref{t:Rprog_application} can be found in \cite{latsze19} (see Corollary~4 and Lemma~7), and the proof of \eqref{e:MIDS_reg_bd} in \cite{rusvan18l} (see Corollary~1 and Proposition~2). A proof overview is provided in Appendix~\ref{s:proof_overview}.

It is clear from Proposition~\ref{t:Rprog_application} that the bound for Tsallis-IDS exhibits a more graceful dependence on $K$, compared to Shannon-IDS. To obtain 
a $\sqrt{KT}$ bound for Shannon-IDS, one will need a ${K/\log(K)}$ upper bound on $\gamma_F (\pi^{\mathrm{NDS}})$, so that $\gamma_F (\pi^{\mathrm{NDS}}) \cdot \mathrm{diam}_F(\Delta^{K-1})$ will have an order $K$ upper bound. 
Theorem~\ref{t:R_prog_hp_lb_alpha_nz} in the following section shows that this is not possible.

\section{A Didactic Example}
\label{s:counter}

The goal here is to show that it is impossible to improve the regret bound for Shannon-IDS using the current analysis. We will also observe that these bounds do not reflect the true behavior of the algorithm in practice, and identify a possible explanation for this.

From here on, unless otherwise mentioned, $F$ refers to the negentropy potential. We will also restrict our theoretical results to the following family of bandit problems that has been a standard for proving lower bounds for bandit algorithms: see for example the proof of Theorem~5.1 in \cite{auecesfresch02} and also the proof of Theorem~1 in \cite{mantsi04}.

\begin{example}
\label{e:hardproblems}
Let $ p \in (0,1)$ and $\epsy \in(0,1-p)$. For $1 \leq i \leq K$, denote $p_{*}^i \in \clP$ such that, for each $1 \leq k \leq K$, the marginals are
\begin{equation}
p_{*}^i(k) = 
\begin{cases}
\mathrm{Bernoulli}(p + \epsy) \, \qquad & \text{if} \,\, i = k
\\
\mathrm{Bernoulli}(p) \, \qquad & \text{if} \,\, i \neq k
\end{cases}
\label{e:p_star_counter}
\end{equation}
$P_*$ is a probability distribution on $\clP$ such that
\begin{equation}
P_*(p_{*}^i) = {1}/{K} \,, \qquad 1 \leq i \leq K
\label{e:P_star_counter}
\end{equation}
\label{ex:counter}
\end{example}
The following result establishes a lower bound on the Shannon information ratio using Example~\ref{e:hardproblems}.
\begin{theorem}
Consider the bandit problem in Example~\ref{ex:counter}. For all $K \geq 2$, there exists $t \geq 0$, $p \in (0,1)$, and $\epsilon^* \in (0,1-p)$, such that, for all $\epsilon \in (0,\epsilon^*]$, and any policy $\pi$,
\begin{equation}
\Gamma_t^F (\pi)
\geq 
{\frac{K}{20}} \qquad a.s..
\label{e:Gamma_t_lb_main}
\end{equation}
\label{t:R_prog_hp_lb_alpha_nz}
\end{theorem}
The proof of Theorem~\ref{t:R_prog_hp_lb_alpha_nz} is contained in Appendix \ref{s:lower_bound_proof}. This result shows that the upper bounds in the second row of Table~\ref{tab:pot_gamma} are tight, up to a constant scaling. 

Theorem \ref{t:R_prog_hp_lb_alpha_nz} further  implies that it is impossible to get an order $\sqrt{KT}$ bound for Shannon-IDS using Corollary~\ref{t:Rprog_uniform_bd}: since $\gamma_F(\pi^{\mathrm{NDS}}) \geq {K}/{20}$, and $\mathrm{diam}_F(\Delta^{K-1}) = \log(K)$,
\[
\sqrt{ \gamma_F (\pi^{\mathrm{NDS}}) \cdot {\mathrm{diam}_F(\Delta^{K-1})} \cdot T } \geq \sqrt{\frac{1}{20}KT\log(K)}
\]
On the other hand, the same line of analysis yields a $\sqrt{2KT}$ bound for Tsallis-IDS.

The gap between the upper bounds of Shannon-IDS and Tsallis-IDS raises the immediate question of whether it reflects real difference in practical performance of the algorithms. To obtain some intuition, we conduct simulations by applying  Thompson sampling, Shannon-IDS and Tsallis-IDS to the bandit problem in Example \ref{e:hardproblems}. 
\begin{figure}[ht]
\begin{center}
{\includegraphics[trim={0cm 5.6cm 0.5cm 5.6cm}, width=0.8\textwidth]{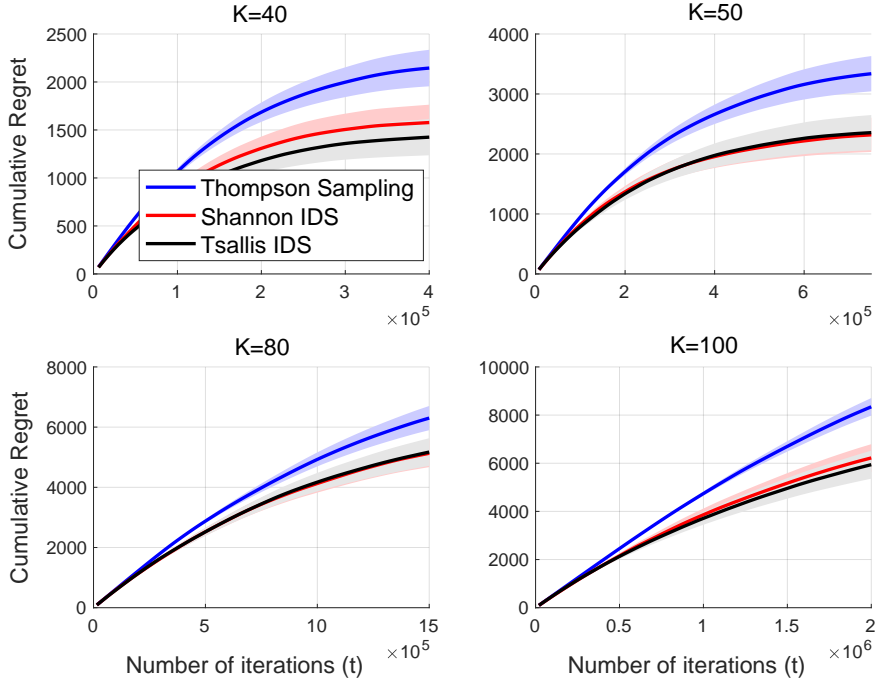}
}
\caption{Expected cumulative Regret of Thompson sampling, Shannon-IDS and Tsallis-IDS applied to Example~\ref{ex:counter}. The shaded regions indicate $2 \sigma$ confidence intervals.}
\label{fig:counter_cum_regret}
\end{center}
\vspace{-0.1in}
\end{figure}

In Figure~\ref{fig:counter_cum_regret} we plot the expected cumulative regret of the three algorithms for $K \in \{40, 50, 80, 100\}$. The details of the experiments are postponed to Section~\ref{s:comp_results}, but the key observation we make is that there is little difference between the performances of Shannon-IDS and Tsallis-IDS. We certainly don't observe the kind of performance gap suggested by Proposition~\ref{t:Rprog_application}: For $K=100$, and $T = 2 \times 10^6$, \eqref{e:HIDS_reg_bd} and \eqref{e:MIDS_reg_bd} suggest a performance gap of
\[
\sqrt{KT} \left (\sqrt{\frac{1}{2} \log(K)}  - \sqrt{2} \right )= 1.5 \times 10^3
\]
which is clearly not the case.

\begin{figure}[ht]
\centering
\begin{center}
{\includegraphics[trim={1cm 5.6cm 0 5.6cm}, width=0.8\textwidth]{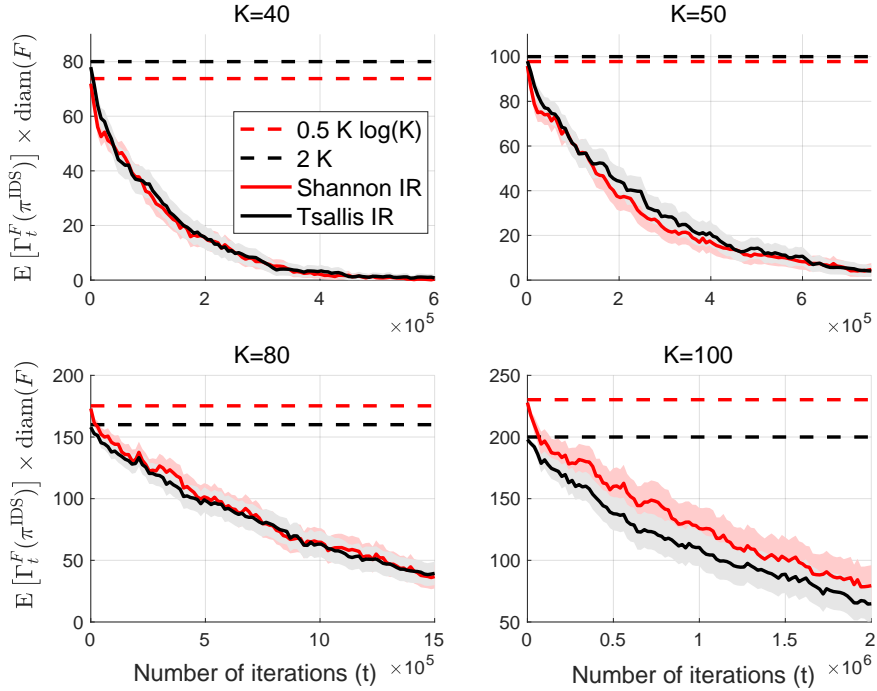}
}
\caption{Empirical average and $2 \sigma$ confidence intervals of $\Gamma_t^F(\pi^{\mathrm{IDS}})$ corresponding to Example~\ref{ex:counter}.}
\label{fig:counter_rv_info_ratio_IDS}
\end{center}
\vspace{-0.2in}
\end{figure}

To understand the mystery behind the discrepancy between the theory and practical performance, in Figure~\ref{fig:counter_rv_info_ratio_IDS} we plot the empirical average of $\Gamma_t^F(\pi^{\mathrm{NDS}})$ and $\Gamma_t^F(\pi^{\mathrm{TDS}})$ scaled by $\mathrm{diam}_F(\Delta^{K-1})$. The scaling makes sure that the units corresponding to the two plots match. The dashed plots indicate the worst case bounds on the information ratio for each of the two algorithms.

The plots in Figure~\ref{fig:counter_rv_info_ratio_IDS} identify a plausible explanation for the gap in theory and practice: The theoretical results use worst case information ratio to bound the regret, but $\Gamma_t^F$ is clearly a time-varying quantity. More importantly, even though at $t=0$, $\Expect[\Gamma_t^F(\pi^{\mathrm{IDS}})]$ closely matches the worst-case bounds, it is monotonically decreasing for $t > 0$. This suggests that it is crucial to take into account the temporal nature of the information ratio in analysis. And based on Figure~\ref{fig:counter_rv_info_ratio_IDS}, an application of Theorem~\ref{t:Rprog} will predict similar performance bounds for Shannon-IDS and Tsallis-IDS, consistent with our observation in Figure~\ref{fig:counter_cum_regret}.

\section{Accounting for Temporal Variation of Information Ratio}
\label{s:time_var_IR}

The lower bound we established in Section~\ref{s:counter} implies that it is not possible to obtain an order $\sqrt{KT}$ regret bound for Shannon-IDS using a particular template for analysis, which depends on the information ratio through its maximum over time.  Our experimental results suggest it may be possible to establish such a bound via an analysis that accounts for temporal variation of the information ratio.

In this section, we propose a new template for analysis.  We 
will use this to obtain an order $\sqrt{KT}$ bound for Thompson sampling via studying the Shannon information ratio, and in particular, its time variation.  Similarly with Shannon-IDS, the bound for Thompson sampling can not be established without taking this time variation into account.

Theorem~\ref{t:RprogExt} below can be regarded as a generalization of Corollary~\ref{t:Rprog_uniform_bd}. The key difference between the two results is that Theorem~\ref{t:RprogExt} accounts for the time-varying nature of information ratio, whereas Corollary~\ref{t:Rprog_uniform_bd} does not.

\begin{theorem}
\label{t:RprogExt}
For all $T \in \N$, $K\geq2$, $\bfmdelta = \{\delta_t: \delta_t \in [0,1] \,, 0 \leq t \leq T - 1\}$, $\gamma_{FX} \in \Re^+ \cup \{\infty\}$, and policies $\pi$, such that
\[
\Pr\left (\Gamma_t^F (\pi)  \leq \gamma_{FX} \right) \geq 1 - \delta_t\,, \quad t = 0, \ldots , T-1
\]
we have
\begin{align}
{\mathrm{Regret}}(T, \pi) \! \leq \! \sqrt{ \gamma_{FX} \! \cdot \! {\mathrm{diam}_F(\Delta^{K-1})} \cdot T } 
+ \epsy \! \cdot \! \sum_{t=0}^{T-1} \delta_t
\nonumber
\end{align}
\end{theorem}
The proof of Theorem~\ref{t:RprogExt} is contained in Appendix~\ref{s:RprogExt_proof}. Note that the result coincides with Corollary~\ref{t:Rprog_uniform_bd} in the special case $\delta_t = 0$, $t = 0,\ldots,T-1$.

Proposition~\ref{t:sqrtKTbd_hp} is an application of Theorem~\ref{t:RprogExt}, and establishes an order $\sqrt{KT}$ regret bound for Thompson sampling via an analysis of the Shannon information ratio.
\begin{proposition}
For all $K \geq 2$, and $t \geq 0$, 
\begin{equation}
\Pr\left (\Gamma_t^F (\pi^{\mathrm{TS}})  \leq 8 \right) \geq 1 - \frac{1}{\epsy} \cdot \sqrt{\frac{8K}{t}}
\label{e:HprobBd_TS}
\end{equation}
Consequently, for all $T \in \N$, $K \geq 2$,
\begin{equation}
{\mathrm{Regret}}(T, \pi^{\mathrm{TS}}) \leq 6 \sqrt{2 KT}
\label{e:R_prog_hp_ExT_ub}
\end{equation}
\label{t:sqrtKTbd_hp}
\end{proposition}
The proof of \eqref{e:HprobBd_TS} is presented in Section~\ref{s:sqrtKTbd_hp_proof} of the appendix. The regret bound \eqref{e:R_prog_hp_ExT_ub} then follows from Theorem~\ref{t:RprogExt}, by letting $\delta_t = \frac{1}{\epsy} \cdot \sqrt{\frac{8K}{t}}$, $\pi = \pi^{\mathrm{TS}}$, and $\gamma_{FX} = 8$. 

\section{Computational Results}
\label{s:comp_results}

In this section we present numerical results that reiterate the key points of our main theoretical results. We consider two sets of experiments: (i) Example~\ref{ex:counter} that was used to establish a lower bound in Theorem~\ref{t:R_prog_hp_lb_alpha_nz}, and (ii) the Beta-Bernoulli setting. 

In each of the settings, we compare performances of three algorithms in terms of their expected cumulative regret:
(i) Thompson sampling that assigns action probabilities according to \eqref{e:TS_prob_matching}, 
(ii) Shannon-IDS: \eqref{e:IDS} with $F$ defined in \eqref{e:negentropy}, and 
(iii) Tsallis-IDS: \eqref{e:IDS} with $F$ defined in \eqref{e:halftsallis}. 

In addition to comparing algorithm performance, we also estimate and compare Shannon information ratio and Tsallis information ratio in each experiment. This gives us an under-the-hood view of each algorithm.

\subsection{Experimental Results for Example~\ref{ex:counter}}
\label{s:sim_counter}

Recall Example~\ref{ex:counter} defined in \eqref{e:p_star_counter} and \eqref{e:P_star_counter}. 

We show results for $K\in\{40, 50, 80, 100\}$, and in each case, we let $p = 1/2$, and $\epsilon = p/K$. The total number of time-steps $T$ for each experiment was chosen to satisfy
$$
T \geq \frac{2}{5} \frac{K^4}{K-1} \log(K)
$$
These choices of $p$, $\epsilon$, and $T$ satisfy the conditions required to establish a stronger lower bound than the one in Theorem~\ref{t:R_prog_hp_lb_alpha_nz} (details are contained Appendix~\ref{s:lower_bound_proof}). Note that for $K=80$ and $100$, the best known regret bound for Shannon-IDS, which is $\sqrt{\frac{1}{2} K T\log(K)}$, is larger than the best known regret bound for Tsallis-IDS, which is $\sqrt{2 K T}$.

At each $t = 0,\ldots, T-1$, an agent selections an action $A_t \in \clA$ and observes $R_{t+1, A_t} \in \{0,1\}$. Given $H_t = \{A_s, R_{s+1, A_s}:s = 0,1,2,\ldots,t-1\}$, we can compute the posterior distribution on the optimal action:  
\[
\Pr_t(A_* = a) = \frac{(1+\epsilon/p)^{s_a(t)} (1-\epsilon/(1-p))^{f_a(t)} }{\sum_{a' \in \clA} (1+\epsilon/p)^{s_{a'}(t)} (1-\epsilon/(1-p))^{f_{a'}(t)}} 
\]
where,
\[
\begin{aligned}
s_a(t) & = \sum_{s=1}^{t-1} R_{s+1,A_s} \cdot \ind\{A_s = a\} 
\\
f_a(t) & = \sum_{s=1}^{t-1} (1-R_{s+1,A_s}) \cdot \ind\{A_s = a\}
\end{aligned}
\]
are the total number of $1$'s and $0$'s observed from arm $a$ at time $t-1$.
Using the above  closed form expressions, it is straightforward to obtain the three algorithms from their definitions -- complete implementation details are in Appendix~\ref{s:exp_details}.

In Figure~\ref{fig:counter_cum_regret} we plot the expected cumulative regret of the three algorithms. The empirical average of the cumulative regret was obtained by simulating $N=200$ independent trajectories for $K\in\{40,50,80\}$, and $N=160$ for $K=100$. The shaded regions indicates $2\sigma$ confidence intervals.

It is clear from these plots that the two IDS algorithms have better performance compared to Thompson sampling. The performance difference between Shannon-IDS and Tsallis-IDS is within a margin of statistical error.

\vspace{-0.3in}
\begin{figure}[ht]
\begin{center}
{\includegraphics[trim={1.4cm 5.4cm 0.5cm 5.4cm}, width=0.8\textwidth]{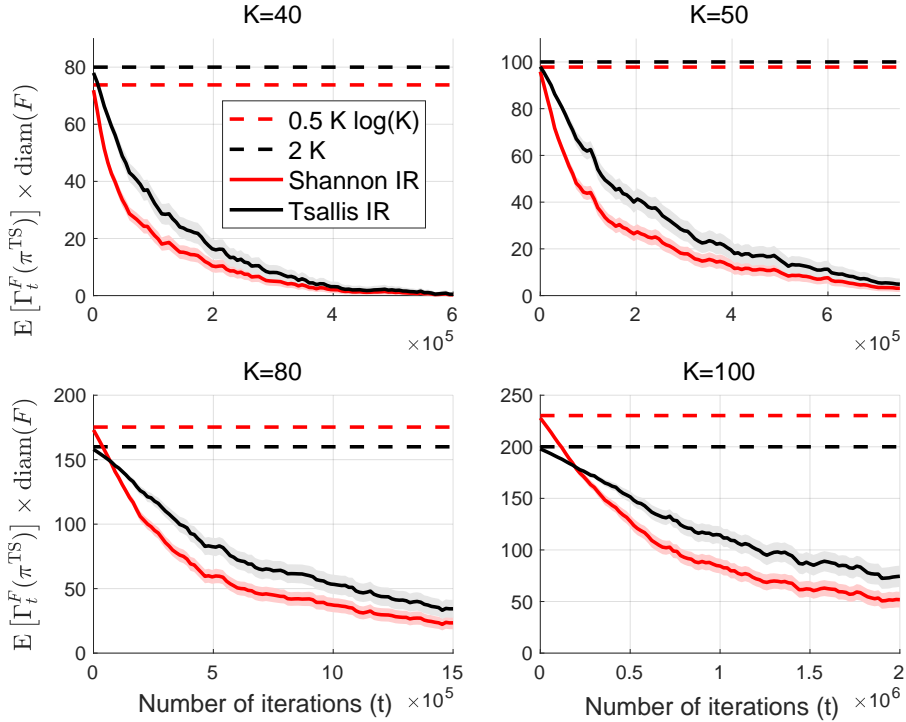}
}
\vspace{-0.1in}
\caption{Empirical average and $2 \sigma$ confidence intervals of $\Gamma_t^F(\pi^{\mathrm{TS}})$ corresponding to the counter Example~\ref{ex:counter}.}
\label{fig:counter_rv_info_ratio_TS}
\vspace{-0.2in}
\end{center}
\end{figure}

In Figure~\ref{fig:counter_rv_info_ratio_IDS} we plot the estimate of $\Expect[ \Gamma_t^F (\pi^{\mathrm{IDS}}) ] \times \mathrm{diam}_F(\Delta^{K-1})$ as a function of $t$, for the two IDS algorithms. The scaling of the information ratio by  $\mathrm{diam}_F(\Delta^{K-1})$ ensures that we are comparing plots with the same units. It is clear that using empirical estimates of $\Expect[ \Gamma_t^F (\pi^{\mathrm{IDS}}) ] \times \mathrm{diam}_F(\Delta^{K-1})$ in Theorem~\ref{t:Rprog} will result in near-identical performance bounds for Shannon-IDS and Tsallis-IDS. This is consistent with our observations in Figure~\ref{fig:counter_cum_regret}. The dashed lines indicate the worst case bounds of $\Gamma_t^F (\pi^{\mathrm{IDS}}) \times \mathrm{diam}_F(\Delta^{K-1})$ which was used to obtain performance bounds in Corollary~\ref{t:Rprog_uniform_bd}. The plots suggest that this will surely lead to looser bounds. More importantly, comparing algorithms based on these looser bounds may lead to a premature conclusion that Tsallis-IDS is better than Shannon-IDS.

In Figure~\ref{fig:counter_rv_info_ratio_TS} we plot the estimate of $\Expect[ \Gamma_t^F (\pi^{\mathrm{TS}}) ] \times \mathrm{diam}_F(\Delta^{K-1})$ for the two cases of $F$ being negentropy potential, and $1/2$-Tsallis entropy. We make similar conclusions as before: Though the worst case upper bound on these quantities can be significantly different for the two different information ratios, when plotted as a function of time, both are observed to be converging to zero. In-fact, we observe that the Shannon information ratio is converging to zero faster than the Tsallis information ratio, implying that an application of Theorem~\ref{t:Rprog} will result in a better regret bound for Thompson sampling, with Shannon information ratio analysis.

\begin{figure}[ht]
\begin{center}
{\includegraphics[trim={1cm 6.7cm 0.5cm 5.3cm}, width=0.8\textwidth]{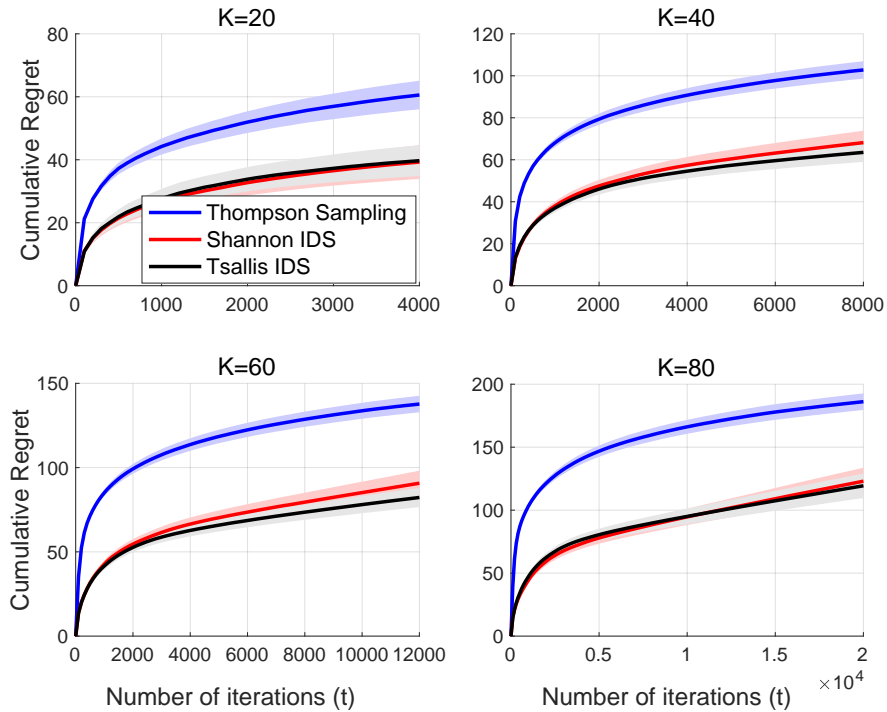}
}
\caption{Expected cumulative Regret of Thompson sampling, Shannon IDS and Tsallis IDS applied to Beta-Bernoulli bandits with $K\in \{20, 40, 60, 80\}$. Shaded regions indicate $2 \sigma$ confidence intervals.}
\label{fig:beta_ber_cum_regret}
\vspace{-0.2in}
\end{center}
\end{figure}

\subsection{Beta Bernoulli}
\label{s:sim_beta_ber}

In our second set of experiments, we consider a $K$-arm bandit problem with independent arms and Bernoulli rewards. Specifically, we consider the Beta-Bernoulli setting, wherein the mean reward for each of the $K$ arms are independently sampled from $\mathrm{Beta}(1,1)$, which is the uniform distribution on $[0,1]$. We show results for $K \in \{20, 40, 60, 80\}$.  

While the implementation of Thompson sampling for this setting is well-known, the implementation of Shannon-IDS and Tsallis-IDS is not straightforward. In-fact, it is practically not possible to exactly compute the information ratio at each time-step, as it involves evaluating integrals that don't have nice closed-form expressions. We can, however, compute approximations of the information ratio, that can be used to obtain approximate versions of the two IDS algorithms. Complete details of implementation are contained in Appendix~\ref{s:exp_details} (also see Section~6.1 of \cite{rusvan18l}; in particular Example~8 and Algorithm~2).

In Figure~\ref{fig:beta_ber_cum_regret} we compare the performances of Thompson sampling, Shannon-IDS and Tsallis-IDS by plotting the expected cumulative regret. The empirical average of the cumulative regret was obtained by running $N=500$ independent runs for each $K$. Once again, we notice that there's little difference in performances of the two IDS algorithms. Both of them are clearly superior to Thompson sampling.

In Figures~\ref{fig:beta_bernoulli_rv_info_ratio_IDS} and~\ref{fig:beta_bernoulli_rv_info_ratio_TS} we plot the estimates of $\Expect[\Gamma_t^F(\pi^{\mathrm{IDS})}]$ and $\Expect[\Gamma_t^F(\pi^{\mathrm{TS}})]$ with $2\sigma$ confidence intervals. We observe that the information ratios decrease much more quickly in this setting, compared to the counter example in Section~\ref{s:sim_counter}. Hence the $\log$-scale for the $Y$-axes. More importantly, the scaled Shannon information ratio is consistently smaller than the scaled Tsallis information ratio, except at $t=0$.

These observations reassert our key point: Accounting for temporal variation of information ratio in regret analysis is crucial for obtaining tight bounds. These new bounds \emph{may} have different implications compared to the existing bounds.

\begin{figure}[htbp]
\begin{center}
{\includegraphics[trim={1cm 9cm 1cm 9.2cm}, width=0.8\textwidth]{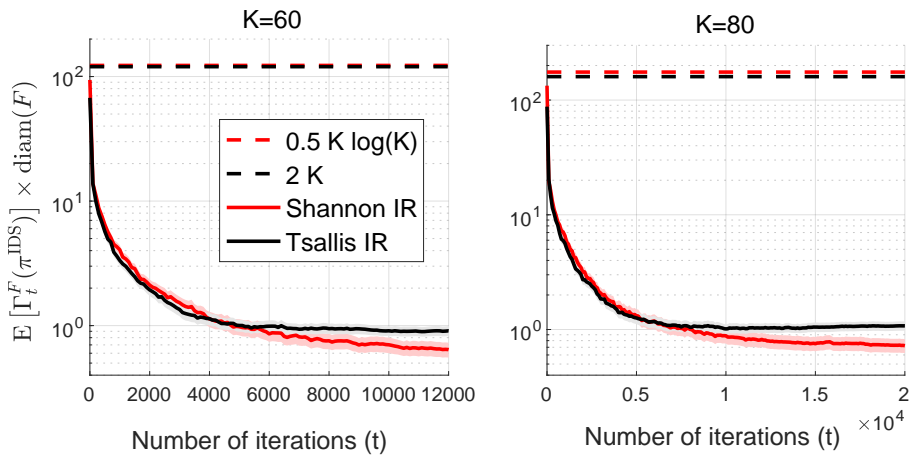}
}
\caption{Empirical average and $2 \sigma$ confidence intervals of $\Gamma_t^F(\pi^{\mathrm{IDS}})$ as a function of $t$ for the Beta-Bernoulli experiment.}
\label{fig:beta_bernoulli_rv_info_ratio_IDS}
\end{center}
\end{figure}

\begin{figure}[htbp]
\begin{center}
{\includegraphics[trim={1cm 9cm 1cm 10cm}, width=0.8\textwidth]{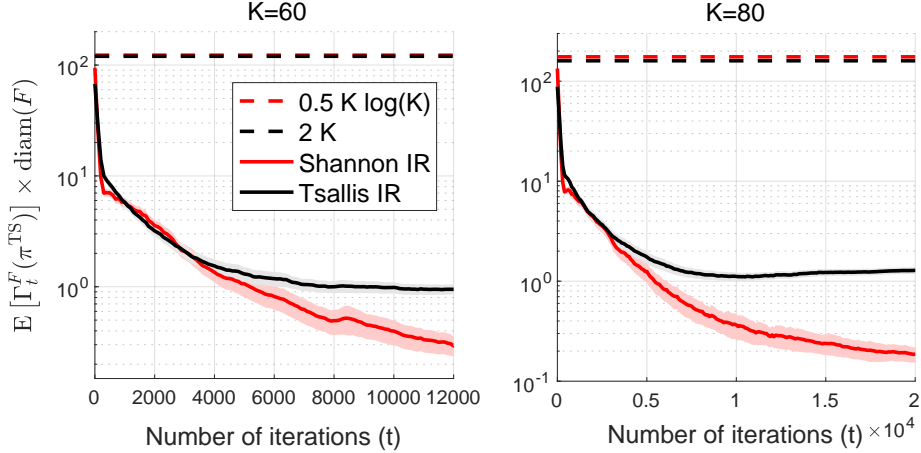}
}
\caption{Empirical average and $2 \sigma$ confidence intervals of $\Gamma_t^F(\pi^{\mathrm{TS}})$ as a function of $t$ for the Beta-Bernoulli experiment.}
\label{fig:beta_bernoulli_rv_info_ratio_TS}
\vspace{-0.1in}
\end{center}
\end{figure}

\section{Example: Sparse Linear Bandits with Non-Uniform Prior}
\label{s:feature_select_ex}

The experimental results we have shown so far do not indicate a clear favourite between Shannon-IDS and Tsallis-IDS. In this section we explore whether it is possible to identify a problem instance where one algorithm is clearly superior to the other. 

Ideally, we would have liked to design a multi-arm bandit problem that exactly falls under our problem formulation of Section~\ref{s:setup}, and then compare the two IDS algorithms on this problem. However, we were unable to find such an example. Instead, we propose here a sparse linear bandit problem, where the agent makes an observation in addition to the reward. We will show that Tsallis-IDS performs strictly worse than Shannon-IDS in this class of problems.

The example we propose is closely related to Example 3 of \cite{rusvan18l}. For simplicity, assume $K = 2^{m} + 1$ for some $m \in \N$. Given $K \geq 3$, let the action set $\clA = \{0,1\}^K$. For $1 \leq i \leq K$, denote $e_i$ to be the $i^{\text{th}}$ basis vector: $e_i \in \Re^K$, $e_i(k) = \ind_{i = k}$.
The prior is assumed to be non-uniform: 
\[
P_*(e_i) = 
\begin{cases}
\frac{1}{2} \, \qquad & \text{if} \,\, \quad i = 1
\\
\frac{1}{2(K-1)} \, \qquad & \text{if} \,\, \quad i \neq 1
\end{cases}
\]

Upon choosing an action $a \in \clA$, the agent receives a reward $R_{t+1,a}$ at time-step $t$, where
\[
R_{t+1,a} = 
\begin{cases}
1 \, \qquad & \text{if} \quad a = p_*
\\
-1 \, \qquad & \text{if} \quad a = e_i \,, 1 \leq i \leq K, a \neq p_*
\\
0 \, \qquad & \text{otherwise}
\end{cases}
\]
In addition to receiving a reward at each time-step, the agent also observes $y = a^\intercal p_*$.

Since the optimal action is $p_*$, the goal of any optimal agent should be to identify $p_*$ as quickly as possible. Therefore, it is obvious that the optimal action in the first iteration is either $a = e_1$ or $a  = [0,1,\ldots,1]$. Either of these actions will reveal the true parameter $p_*$ with probability $1/2$, and with the other $1/2$ probability, $p_*\in \{e_2, \ldots, e_K\}$. In addition, both these actions result in the same expected instantaneous regret of $0$. 

Suppose $p_*$ is not revealed in the first iteration, the optimal sequence of actions from the second iteration on-wards is to perform a binary search. That is, in the second iteration, first half of the last $K-1$ components of $a$ are chosen to be ones, and the second half of the last $K-1$ components are chosen to be zeros, and the process repeats until $p_*$ is identified. In the worst case, the total number of iterations required to find the optimal action via this procedure is $m + 1$ iterations.

Simple computations show that the Shannon-IDS algorithm precisely follows these steps. This is due to the fact that these sequence of actions result in maximum expected reduction in Shannon entropy at each iteration. On the other hand, simple computation shows that for $K \geq 5$, picking a sub-optimal action $a=[1,1,0,\ldots,0]$ in the first iteration results in greater expected reduction in $1/2$-Tsallis entropy compared to action $a=[1,0,\ldots,0]$ or $a=[0,1,\ldots,1]$. Since all these actions incur $0$ expected reward, the Tsallis-IDS algorithm \emph{does not} choose one of the two optimal actions at $t = 0$. This implies that the Tsallis-IDS algorithm provably takes a greater number of iterations to identify $p_*$ in expectation.

An interesting observation we make is that the Thompson sampling agent will require order $K$ iterations in expectation to identify $p_*$. This is because it assigns non-zero probabilities only to actions $a \in \{e_1, e_2, \ldots, e_K\}$, and rules out a single action in each iteration.

\section{Conclusion}

While the sparse linear bandit problem in Section~\ref{s:feature_select_ex} was very much a stylized example, it demonstrates the existence of a class of problems where quantifying information via Shannon entropy is clearly better. 

We searched for a problem where we could demonstrate the advantage of using Tsallis-IDS over Shannon-IDS in a similar manner, but we did not succeed. In fact, the original motivation for considering Example~\ref{ex:counter} was to design a hard problem where Shannon-IDS will fail. However, despite the gap in the performance bounds, we showed that in practice, the realized performance gap between the two algorithms is negligible. Whether there exists a bandit problem where Tsallis-IDS is provably better than Shannon-IDS remains an open question.

Additional theoretical and computational results identified a plausible explanation for the gap in the bounds: Existing techniques that upper bound the regret use worst case bounds on the information ratio, however, this quantity is highly time-varying, and taking into account this property is crucial to obtaining tighter bounds.  

While the results presented in this paper are preliminary, it opens up a lot of avenues for future research. Few of them are listed below.

\begin{romannum}
\item A challenging work for the future is to generalize the order $\sqrt{KT}$ bound in Proposition~\ref{t:sqrtKTbd_hp} for any $P_*$, for both Thompson sampling and Shannon-IDS. This result, if true, will close the gap in performance bounds for the two IDS algorithms.

\item Our paper restricted to the case wherein the information gain (and consequently the information ratio) was defined with respect to the posterior on the optimal action. Extension of the results to general continuous action spaces, such as general linear bandits requires defining a notion of \emph{satisficing} action, that can be thought of as an approximation to the optimal action that is easier to learn. In this set-up, the information gain is defined with respect to the satisficing action \cite{rusvan18s,donvan18}. While the rate-distortion theory provides natural tools for analysis Shannon-IDS in this framework, it is interesting to find out if an analog exists for Tsallis-IDS.

\item An interesting future work is figuring out how to automate the choice of information gain function depending on the application. While we have still not identified a problem where Tsallis-IDS is provably better, the existence of such a problem can not be ruled out, and a generic algorithm that adapts the information gain function according to a specific application may be extremely useful in practice. 
\end{romannum}

\newpage

\bibliography{On_Info_Measures_ICML}

\begin{thebibliography}{10}

\bibitem{agrgoy13}
S.~Agrawal and N.~Goyal.
\newblock Further optimal regret bounds for thompson sampling.
\newblock In {\em Artificial intelligence and statistics}, pages 99--107. PMLR,
  2013.

\bibitem{audbub09}
J.-Y. Audibert and S.~Bubeck.
\newblock Minimax policies for adversarial and stochastic bandits.
\newblock In {\em COLT}, pages 217--226, 2009.

\bibitem{auecesfresch02}
P.~Auer, N.~Cesa-Bianchi, Y.~Freund, and R.~E. Schapire.
\newblock The nonstochastic multiarmed bandit problem.
\newblock {\em SIAM journal on computing}, 32(1):48--77, 2002.

\bibitem{bubaud12}
S.~Bubeck, N.~Cesa-Bianchi, et~al.
\newblock Regret analysis of stochastic and nonstochastic multi-armed bandit
  problems.
\newblock {\em Foundations and Trends{\textregistered} in Machine Learning},
  5(1):1--122, 2012.

\bibitem{bubdekkorper15}
S.~Bubeck, O.~Dekel, T.~Koren, and Y.~Peres.
\newblock Bandit convex optimization:$\backslash$sqrtt regret in one dimension.
\newblock In {\em Conference on Learning Theory}, pages 266--278. PMLR, 2015.

\bibitem{bubliu14}
S.~Bubeck and C.-Y. Liu.
\newblock Prior-free and prior-dependent regret bounds for thompson sampling.
\newblock In {\em 2014 48th Annual Conference on Information Sciences and
  Systems (CISS)}, pages 1--9. IEEE, 2014.

\bibitem{donvan18}
S.~Dong and B.~Van~Roy.
\newblock An information-theoretic analysis for thompson sampling with many
  actions.
\newblock In {\em Advances in Neural Information Processing Systems}, pages
  4157--4165, 2018.

\bibitem{kirkra18}
J.~Kirschner and A.~Krause.
\newblock Information directed sampling and bandits with heteroscedastic noise.
\newblock In {\em Conference On Learning Theory}, pages 358--384. PMLR, 2018.

\bibitem{kirlatkra20}
J.~Kirschner, T.~Lattimore, and A.~Krause.
\newblock Information directed sampling for linear partial monitoring.
\newblock In {\em Conference on Learning Theory}, pages 2328--2369. PMLR, 2020.

\bibitem{kirlatversze20}
J.~Kirschner, T.~Lattimore, C.~Vernade, and C.~Szepesv{\'a}ri.
\newblock Asymptotically optimal information-directed sampling.
\newblock {\em arXiv preprint arXiv:2011.05944}, 2020.

\bibitem{lat20}
T.~Lattimore.
\newblock Improved regret for zeroth-order adversarial bandit convex
  optimisation.
\newblock {\em arXiv preprint arXiv:2006.00475}, 2020.

\bibitem{latgyo20}
T.~Lattimore and A.~Gy{\"o}rgy.
\newblock Mirror descent and the information ratio.
\newblock {\em arXiv preprint arXiv:2009.12228}, 2020.

\bibitem{latsze19}
T.~Lattimore and C.~Szepesv{\'a}ri.
\newblock An information-theoretic approach to minimax regret in partial
  monitoring.
\newblock {\em arXiv preprint arXiv:1902.00470}, 2019.

\bibitem{latsze20b}
T.~Lattimore and C.~Szepesv{\'a}ri.
\newblock {\em Bandit algorithms}.
\newblock Cambridge University Press, 2020.

\bibitem{latsze20e}
T.~Lattimore and C.~Szepesv{\'a}ri.
\newblock Exploration by optimisation in partial monitoring.
\newblock In {\em Conference on Learning Theory}, pages 2488--2515. PMLR, 2020.

\bibitem{luvan19}
X.~Lu and B.~Van~Roy.
\newblock Information-theoretic confidence bounds for reinforcement learning.
\newblock In {\em Advances in Neural Information Processing Systems}, pages
  2461--2470, 2019.

\bibitem{mantsi04}
S.~Mannor and J.~N. Tsitsiklis.
\newblock The sample complexity of exploration in the multi-armed bandit
  problem.
\newblock {\em Journal of Machine Learning Research}, 5(Jun):623--648, 2004.

\bibitem{rusvan14}
D.~Russo and B.~Van~Roy.
\newblock Learning to optimize via information-directed sampling.
\newblock In {\em Advances in Neural Information Processing Systems}, pages
  1583--1591, 2014.

\bibitem{rusvan14l}
D.~Russo and B.~Van~Roy.
\newblock Learning to optimize via posterior sampling.
\newblock {\em Mathematics of Operations Research}, 39(4):1221--1243, 2014.

\bibitem{rusvan16}
D.~Russo and B.~Van~Roy.
\newblock An information-theoretic analysis of thompson sampling.
\newblock {\em The Journal of Machine Learning Research}, 17(1):2442--2471,
  2016.

\bibitem{rusvan18l}
D.~Russo and B.~Van~Roy.
\newblock Learning to optimize via information-directed sampling.
\newblock {\em Operations Research}, 66(1):230--252, 2018.

\bibitem{rusvan18s}
D.~Russo and B.~Van~Roy.
\newblock Satisficing in time-sensitive bandit learning.
\newblock {\em arXiv preprint arXiv:1803.02855}, 2018.

\bibitem{rusvankazosbwen17}
D.~Russo, B.~Van~Roy, A.~Kazerouni, I.~Osband, and Z.~Wen.
\newblock A tutorial on thompson sampling.
\newblock {\em arXiv preprint arXiv:1707.02038}, 2017.

\bibitem{tho33}
W.~R. Thompson.
\newblock On the likelihood that one unknown probability exceeds another in
  view of the evidence of two samples.
\newblock {\em Biometrika}, 25(3/4):285--294, 1933.

\bibitem{zimlat19}
J.~Zimmert and T.~Lattimore.
\newblock Connections between mirror descent, thompson sampling and the
  information ratio.
\newblock {\em arXiv preprint arXiv:1905.11817}, 2019.

\end{thebibliography}
\bibliographystyle{abbrv}

\newpage

\appendix

\appendix
\onecolumn
\begin{center}
{\bf \huge Appendix}
\end{center}

\section*{\Large Overview of the Appendix}

First, we introduce notation that is used throughout the Appendix.

In Section~\ref{s:proof_overview} we give a proof overview of Proposition~\ref{t:Rprog_application}. We also discuss the implications of the result.

In Section~\ref{s:lower_bound_proof}, we provide the proof of Theorem~\ref{t:R_prog_hp_lb_alpha_nz}.

In Section~\ref{s:lower_bound_proof_rand_pol} we strengthen the lower bound of Theorem~\ref{t:R_prog_hp_lb_alpha_nz}. Specifically, we show that the lower bound carries through even with an extended definition of information ratio (see \eqref{e:IR_alpha}) that includes a ``slack'' parameter (as in \cite{rusvan18s,latsze19}). We also show that replacing a uniform (in time) almost sure upper bound on the information ratio, such as the one used in Corollary~\ref{t:Rprog_uniform_bd} (and in previous literature), with a uniform (in time) \emph{high probability} upper bound is insufficient to obtain an order $\sqrt{KT}$ bound for \emph{any} policy, using Shannon information ratio analysis. These results highlight the need for analysis techniques that take into account the temporal nature of information ratio.

In Section~\ref{s:RprogExt_proof} we provide proofs of Theorems~\ref{t:Rprog} and~\ref{t:RprogExt}. We also provide a new template for analysis of the Bayesian regret of any policy in the form of Theorem~\ref{t:RprogExtGen}, which is a generalization of Theorem~\ref{t:RprogExt}. Contrary to previous analysis techniques, the proposed method accounts for he time-varying nature of the information ratio.

In Section~\ref{s:sqrtKTbd_hp_proof} we provide the proof of Proposition~\ref{t:sqrtKTbd_hp}. The proof will use the results of Section~\ref{s:RprogExt_proof} to prove an order $\sqrt{KT}$ bound for Thompson sampling using Shannon information ratio analysis.

Section~\ref{s:exp_details} contains details of our numerical results, as well as some additional experimental results.

\section*{\Large Notations}

For a random variable $X$, we will denote by $P_t(X)$, the probability distribution function of $X$, conditioned on $H_t$: $$P_t(X) \equiv \Pr_t(X \in \cdot)$$ 
Similarly, we denote 
$$
\begin{aligned}
P_t(X|Y) & \equiv \Pr_t(X \in \cdot | Y)
\\
P_t(X|Y = y) & \equiv \Pr_t(X \in \cdot | Y = y)
\end{aligned}
$$

Unless otherwise mentioned, throughout the supplementary material, $F$ is the negentropy potential:
\[
F(u) = \sum_{i\in \clA: u_i > 0} (u_i \log(u_i) - u_i) \,, \qquad u \in \Delta^{K-1}
\]

\newpage 

\section{Proof Overview of  Proposition~\ref{t:Rprog_application}}
\label{s:proof_overview}

\subsection{Proof overview}

It was shown in \cite{latsze19} that the information ratio for Thompson sampling satisfies,
\begin{equation}
\Gamma^F_t(\pi^{\mathrm{TS}}) \leq \sqrt{K} \quad a.s., \quad t = 0,1,\ldots
\label{e:TS_IR_T_upper}
\end{equation}
where $F$ is the $1/2$-Tsallis entropy. Using the fact that ${\mathrm{diam}_F(\Delta^{K-1})} \leq 2 \sqrt{K}$, and applying Theorem~\ref{t:Rprog}, we obtain \eqref{e:TS_reg_bd}. The bound in \eqref{e:HIDS_reg_bd} follows similar arguments: With the same potential $F$, using \eqref{e:IDS} and \eqref{e:TS_IR_T_upper}, we can show that, 
\[
\Gamma^F_t(\pi^{\mathrm{TDS}}) \leq \sqrt{K} \quad a.s., \quad t = 0,1,\ldots
\]
Applying Theorem~\ref{t:Rprog} once again yields \eqref{e:HIDS_reg_bd}.

Deriving the bound in \eqref{e:MIDS_reg_bd} naturally requires consideration of the negentropy potential. It was shown in \cite{rusvan16} that the information ratio for Thompson sampling satisfies (see Proposition~3):
\begin{equation}
\Gamma^F_t(\pi^{\mathrm{TS}}) \leq K/2 \quad a.s., \quad t = 0,1,\ldots
\label{e:TS_IR_H_upper}
\end{equation}
where $F$ is the negentropy potential. With the same potential function, using \eqref{e:IDS} and \eqref{e:TS_IR_H_upper}, it can be shown that (see Proposition~2 of \cite{rusvan18l})
\[
\Gamma^F_t(\pi^{\mathrm{NDS}}) \leq K/2 \quad a.s., \quad t = 0,1,\ldots
\]
Since we have ${\mathrm{diam}_F(\Delta^{K-1})} \! \leq \! \log(K)$ for the negentropy potential, applying Theorem~\ref{t:Rprog} results in \eqref{e:MIDS_reg_bd}.
\qed

\subsection{Comments on the proof}

The information ratio analysis of Thompson sampling was first introduced in \cite{rusvan16}, using the negentropy potential. In comparison to \eqref{e:TS_reg_bd}, the resulting upper bound had an additional $\sqrt{\frac{1}{4}\log(K)}$ factor, as in the right hand side of \eqref{e:MIDS_reg_bd}. It is interesting to note that a change in the definition of information ratio via a change in the information gain function leads to a tighter bound.

\newpage

\section{Proof of a Lower Bound for the Information Ratio}
\label{s:lower_bound_proof}

Here we will show that for the counter example described in Section~\ref{s:counter}, the lower bound in Theorem~\ref{t:R_prog_hp_lb_alpha_nz} holds. 

In Section~\ref{s:lower_bound_proof_rand_pol} we show that our results will hold for a much more generalized definition of the information ratio \eqref{e:IR}, of which the definition considered in \cite{latsze19} is a special case. 

As a first step, we precisely describe the counter example that we use to establish the lower bound.

Fix the number of arms $K$. Let $ p \in (0,1)$, and $\epsy \in (0, \epsy^*]$, where $0 < \epsy^* < 1-p$ satisfies the following inequalities\footnote{Many of the inequalities in \eqref{e:epsy_conditions} can be combined, but we write each of them out to ease verification of the proof.}:
\begin{subequations}
\begin{align}
\epsy^* & < \min\{p, 1-p\}
\label{e:epsy_condition_1}
\\
\left(\frac{\epsy^*}{K} \right)^2& \leq \frac{1}{2} \cdot p (1-p)
\label{e:epsy_condition_2}
\\
\left(\frac{\epsy^*}{K}\right)^2 & \leq \frac{1}{4} \cdot p(1-p)
\label{e:epsy_condition_3}
\\
\frac{\epsy^*}{K} & \leq \frac{1}{4} \cdot \frac{p (1-p)}{|2p - 1|}
\label{e:epsy_condition_4}
\\
\epsy^* & \leq \frac{(K-1)}{K^2} \cdot \frac{p}{(1-p)}
\label{e:epsy_condition_6}
\\
\epsy^* & \leq \frac{(1-p)}{K}
\label{e:epsy_condition_7}
\\
\epsy^* & \leq \frac{(K-1)}{K} \cdot \frac{2}{3} \cdot \frac{(1-p)}{p}
\label{e:epsy_condition_8}
\end{align}
\label{e:epsy_conditions}
\end{subequations} 
For $1 \leq i \leq K$, denote $p_{*,i} \in \clP$ (recall, $\clP$ is the set of all probability measures on $[0,1]^K$) such that, for each $1 \leq k \leq K$, the marginals are
\begin{equation}
p_{*,i}(k) = 
\begin{cases}
\mathrm{Bernoulli}(p + \epsy) \, \qquad & \text{if} \,\, i = k
\\
\mathrm{Bernoulli}(p) \, \qquad & \text{if} \,\, i \neq k
\end{cases}
\label{e:p_star_counter_rep}
\end{equation}
Throughout this section, we let $P_*$ be the probability distribution on $\clP$ such that
\begin{equation}
P_*(p_{*,i}) = \frac{1}{K} \,, \qquad 1 \leq i \leq K
\label{e:P_star_counter_rep}
\end{equation}

In this section, we'll prove the following Proposition. The proof of Theorem~\ref{t:R_prog_hp_lb_alpha_nz} follows directly from this result.

\begin{proposition}
For all policies $\pi$, $K \geq 2$, $p \in (0,1)$, and $\epsy \in (0, \epsy^*]$, the following holds a.s. at $t = 0$: 
\begin{align}
g_t^F(\pi)
& \leq \frac{(K-1)\epsy^2}{K^2}  \cdot \frac{5}{2 p(1-p)}
\label{e:IG_counter}
\\[0.5em]
\Delta_t(\pi) 
& = \frac{\epsy (K-1)}{K}
\label{e:reg_counter}
\end{align}
Consequently,
\begin{equation}
\Gamma_t^F{(\pi)} \geq {{\frac{2}{5}} \cdot p(1-p) \cdot (K-1)}
\label{e:IR_counter}
\end{equation}
\label{t:IR_counter}
\end{proposition}

\begin{proof}[\bf Proof of Theorem~\ref{t:R_prog_hp_lb_alpha_nz}]
The proof follows directly from Proposition~\ref{t:IR_counter} by letting $p = 1/2$ in \eqref{e:IR_counter}, and using the fact that $K \geq 2$. 
\end{proof}

\subsection*{Proof of Proposition~\ref{t:IR_counter}}

The proof relies on Lemmas~\ref{t:Post_R_largeT} and~\ref{t:IG_Regret_Prop2_RVR}.
\begin{lemma}
For each $t \geq 0$,
\begin{subequations}
\begin{align}
\Pr_t(R_{t+1,a} = 0 | A_* = a_* ) 
& = 
\begin{cases}
1 - (p + \epsy) \,\, & \text{If} \,\, a_* = a
\\
1 - p \,\, & \text{If} \,\, a_* \neq a
\end{cases} 
\label{e:Post_R_counter_0}
\\[0.7em]
\Pr_t(R_{t+1,a} = 1 | A_* = a_* ) 
& = 
\begin{cases}
p + \epsy \,\, & \text{If} \,\, a_* = a
\\
p \,\, & \text{If} \,\, a_* \neq a
\end{cases}
\label{e:Post_R_counter_1}
\\[1em]
\Pr_t(R_{t+1,a} = 0) 
& = 
1 - p - \epsy \cdot \Pr_t(A_* = a)
\label{e:Post_R_marginals_0}
\\[0.7em]
\Pr_t(R_{t+1,a} = 1)
& = 
p + \epsy \cdot \Pr_t(A_* = a)
\label{e:Post_R_marginals_1}
\end{align}
\end{subequations}
\label{t:Post_R_largeT}
\end{lemma}
\begin{proof}
Expressions \eqref{e:Post_R_counter_0} and \eqref{e:Post_R_counter_1} are straightforward from the definition of Example~\ref{ex:counter}. The marginals are also straightforward to compute:
\begin{equation}
\begin{aligned}
\Pr_t(R_{t+1,a} = 0) &= \sum_{a_* \in \clA} \Pr_t(A_* = a_*) \cdot  \Pr_t(R_{t+1,a} = 0 | A_* = a_* ) 
\\
& = (1-p) - \Pr_t(A_* = a) \cdot \epsy 
\\
\Pr_t(R_{t+1,a} = 1) &= \sum_{a_* \in \clA} \Pr_t(A_* = a_*) \cdot  \Pr_t(R_{t+1,a} = 1 | A_* = a_*) 
\\
& = p + \Pr_t(A_* = a) \cdot \epsy
\end{aligned}
\label{e:Post_R_marginals_gen}
\end{equation}
\end{proof}

The following result is an extension of Proposition~2 of \cite{rusvan16} which considered the special case of Thompson sampling agent.
\begin{lemma}
For any policy $\pi$,
\begin{align}
g_t^F(\pi)
& = \sum_{a_*, a \in \clA} \pi(a) \cdot \Pr_t(A_* = a_*) \Big[ D_{{\mathrm{KL}}} \Big(P_t \big( R_{t+1, a} | A_* =  a_* \big) || P_t \big( R_{t+1,a} \big) \Big) \Big]
\label{e:IG_Prop2_RVR}
\\
\Delta_t(\pi)
& = \sum_{a \in \clA} \Pr_t(A_* = a) \cdot \E_t \big[R_{t+1, a} | A_* \! = \! a \big] -  \sum_{a \in \clA}  \pi(a) \cdot \E_t \big[ R_{t+1,a} \big]
\label{e:Regret_Prop2_RVR}
\end{align}
\label{t:IG_Regret_Prop2_RVR}
\end{lemma}
The proof of Lemma~\ref{t:IG_Regret_Prop2_RVR} follows exactly along the lines of proof of Proposition~2 in \cite{rusvan16}, except that $\Pr_t(A_t = a) = \Pr_t(A_* = a)$ in their proof (which holds for Thompson sampling) is replaced by $\Pr_t(A_t = a) = \pi(a)$. 

We are now ready to prove Proposition~\ref{t:IR_counter}.

\begin{proof}[\bf Proof of Proposition~\ref{t:IR_counter}]

Recall that at $t = 0$, $\Pr_t(A_* = a) = 1/K$ for each $a \in \clA$. It follows from \eqref{e:Post_R_marginals_0} and \eqref{e:Post_R_marginals_1} of Lemma~\ref{t:Post_R_largeT} that
for each $a \in \clA$,
\begin{equation}
\begin{aligned}
\Pr_t(R_{t+1,a} = 0) & = (1-p) - \frac{\epsy}{K} 
\\
\Pr_t(R_{t+1,a} = 1)  & = p + \frac{\epsy}{K}
\end{aligned}
\label{e:Post_R_marginals}
\end{equation}
Using \eqref{e:Post_R_counter_0}, \eqref{e:Post_R_counter_1} of Lemma~\ref{t:Post_R_largeT} and \eqref{e:Post_R_marginals} in the expression for relative entropy \eqref{e:KL}, we obtain:
\begin{align}
D_{{\mathrm{KL}}} \Big(P_t \big( R_{t+1, a} | A_* \! = \! a_* \big) || P_t \big( R_{t+1,a} \big) \Big)
& \! = \!
\begin{cases}
(1 \! - \! p \! - \! \epsy) \log \left(\frac{1 - p - \epsy}{1 - p - \epsy / K} \right) + (p \! + \! \epsy) \log \left( \frac{p + \epsy}{p + \epsy / K} \right) & \,\,\, \text{If} \,\, a = a_*
\\[0.5em]
(1 \! - p) \log \left(\frac{1 - p}{1 - p - \epsy / K} \right) + p \log \left(\frac{p}{p + \epsy / K}\right) & \,\,\, \text{If} \,\, a \neq a_*
\end{cases}
\label{e:KL_counter}
\end{align}
Next, using the fact that $\Pr_t(A_* = a_*) = 1/K$ for each $a_* \in \clA$, we have
\begin{equation}
\begin{aligned}
& \sum_{a_* \in \clA} \Pr_t(A_* = a_*) \Big[ D_{{\mathrm{KL}}} \Big(P_t \big( R_{t+1, a} | A_* =  a_* \big) || P_t \big( R_{t+1,a} \big) \Big) \Big] 
\\
& \hspace{1in}= \Big( \frac{K-1}{K} \Big) \left( (1 \! - p) \log \left(\frac{1 - p}{1 - p - \epsy / K} \right) + p \log \left(\frac{p}{p + \epsy / K}\right) \right)
\\
& \hspace{1.2in}+ \frac{1}{K} \left( (1 \! - \! p \! - \! \epsy) \log \left(\frac{1 - p - \epsy}{1 - p - \epsy / K} \right) + (p \! + \! \epsy) \log \left( \frac{p + \epsy}{p + \epsy / K} \right) \right)
\end{aligned}
\label{e:KL_counter_exp}
\end{equation}
Note that the right hand side of \eqref{e:KL_counter_exp} does not depend on $a$ anymore. Substituting \eqref{e:KL_counter_exp} into \eqref{e:IG_Prop2_RVR} of Lemma~\ref{t:IG_Regret_Prop2_RVR}, we have, for any policy $\pi$ (since $\sum_{a\in\clA}\pi(a) = 1$), 
\[
\begin{aligned}
g_t^F(\pi)
= & \Big( \frac{K-1}{K} \Big) \left( (1 \! - p) \log \left(\frac{1 - p}{1 - p - \epsy / K} \right) + p \log \left(\frac{p}{p + \epsy / K}\right) \right)
\\
& + \frac{1}{K} \left( (1 \! - \! p \! - \! \epsy) \log \left(\frac{1 - p - \epsy}{1 - p - \epsy / K} \right) + (p \! + \! \epsy) \log \left( \frac{p + \epsy}{p + \epsy / K} \right) \right)
\end{aligned}
\]
A simplification of the right hand side yields:
\begin{equation}
\begin{aligned}
g_t^F(\pi)
& = -(1 \! - \!p) \log \left( 1 \! - \! \frac{\epsy}{K (1-p)} \right) \! - \! p \log \left(1 \! + \! \frac{\epsy}{Kp} \right)
\\
& + \frac{1}{K} \left( (1 \! - \! p) \log \left( 1 \! - \! \frac{\epsy}{1 \! - \!p} \right) \! + \! p \log \left(1 \! + \! \frac{\epsy}{p} \right) \right)
\\
& + \frac{1}{K} \Big( \epsy \log \left( p + \epsy \right) - \epsy \log \left( 1 - p - \epsy \right) + \epsy \log \left(1 - p - \epsy/K \right) - \epsy \log \left( p + \epsy / K \right) \Big)
\end{aligned}
\label{e:ItF_counter_itm}
\end{equation}

We use the following inequalities that are obtained using Taylor series to upper bound each of the terms in \eqref{e:ItF_counter_itm}: For $0 < x < 1$,
\begin{subequations}
\begin{align}
\log(1+x) & \leq x - \frac{x^2}{2} + \frac{x^3}{3}
\label{e:log_approx_1}
\\
- \log(1+x) & \leq - x + \frac{x^2}{2}
\label{e:log_approx_2}
\\
\log(1 - x) & \leq - x - \frac{x^2}{2}
\label{e:log_approx_3}
\\
- \log(1 - x) & \leq x + \frac{x^2}{2} + \frac{x^3}{3 (1-x)}
\label{e:log_approx_4}
\end{align}
\end{subequations} 
We can upper bound the first two terms in \eqref{e:ItF_counter_itm} using \eqref{e:log_approx_2}  and \eqref{e:log_approx_4} (here we use condition \eqref{e:epsy_condition_1} for $\epsy$)
\begin{align}
-(1 \! - \!p) \log \left( 1 \! - \! \frac{\epsy}{K (1-p)} \right) \! - \! p \log \left(1 \! + \! \frac{\epsy}{Kp} \right) 
& \leq
(1 \! - \!p) \left( \frac{\epsy}{K (1-p)}
\! + \! \frac{\epsy^2}{2 K^2 (1-p)^2} 
\! + \! \frac{\epsy^3}{3 K^3 (1-p)^3} \cdot \Big(1 - \frac{\epsy}{K(1-p)} \Big)^{-1}
\right)
\nonumber
\\
& \qquad \qquad \! - \! p \left(\frac{\epsy}{Kp} - \frac{\epsy^2}{2 K^2 p^2} \right) 
\nonumber
\\
& =
\frac{\epsy^2}{2 K^2 (1-p)} 
\! + \! \frac{\epsy^3}{3 K^2 (1-p)} \cdot \frac{1}{K(1-p) - \epsy}
\! + \! \frac{\epsy^2}{2 K^2 p}
\nonumber
\\
& \leq \frac{\epsy^2}{2K^2 p (1-p)} \! + \! \frac{\epsy^3}{3 K^2 (K-1) (1-p)^2}
\label{e:ItF_counter_itm_f2}
\end{align}
where in the last inequality we have used $\epsy \leq (1-p)$.
Along similar lines, we can bound the second two terms in \eqref{e:ItF_counter_itm} using \eqref{e:log_approx_1} and \eqref{e:log_approx_3}:
\begin{align}
\frac{1}{K} \left( (1 \! - \! p) \log \left( 1 \! - \! \frac{\epsy}{1 \! - \!p} \right) \! + \! p \log \left(1 \! + \! \frac{\epsy}{p} \right) \right) 
\leq \frac{- \epsy^2 }{2K p (1-p)} + \frac{\epsy^3}{3p^2}
\label{e:ItF_counter_itm_s2}
\end{align}

Now, consider the last four terms of \eqref{e:ItF_counter_itm}:
\begin{align}
\frac{1}{K} \Big( \epsy \log \left( p \! + \! \epsy \right) \! & - \! \epsy \log \left( 1 \! - \! p \! - \! \epsy \right) \! + \! \epsy \log \left(1 \! - \! p \! - \! \epsy/K \right) \! - \! \epsy \log \left( p \! + \! \epsy / K \right) \Big)
\nonumber
\\
&
= \frac{1}{K} \Big(\epsy \log \left( 1 \! + \! \frac{(K-1)\epsy}{K (p \! + \! \epsy / K)} \right) 
- \epsy \log \left( 1 \! - \! \frac{(K-1)\epsy}{K (1 \! - \! p \! - \! \epsy / K)} \right)
\nonumber
\\
&
\overset{(a)}{\leq} \frac{1}{K} \left(  \frac{(K-1)\epsy^2}{K (p \! + \! \epsy / K)} 
\! + \! \frac{(K-1)\epsy^2}{K (1 \! - \! p \! - \! \epsy / K)}
\! + \! \frac{(K-1)^2 \epsy^3}{2 K^2 (1 \! - \! p \! - \! \epsy / K)^2} \cdot \left(1 - \frac{(K-1) \epsy}{K(1-p-\epsy/K)} \right)^{-1}
\right)
\nonumber
\\
&
= \frac{(K-1)\epsy^2}{K} \left(  \frac{1}{K p \! + \! \epsy}
+ \! \frac{1}{K \! - \! K p \! - \! \epsy}  \right) 
\! + \! \frac{(K-1)^2 \epsy^3}{2 K^2 (1 \! - \! p \! - \! \epsy / K)} \cdot \frac{1}{K(1-p-\epsy)}
\nonumber
\\
&
\overset{(b)}{\leq} \frac{(K-1)\epsy^2}{K^2} \cdot \frac{1}{(p \! + \! \epsy / K )(1 \! - \! p \! - \! \epsy / K)} \! + \! \frac{(K-1)^2 \epsy^3}{2 K^3 (1 \! - \! p \! - \! \epsy)^2}
\label{e:ItF_counter_itm_l4}
\end{align}
where $(a)$ uses \eqref{e:log_approx_1} and \eqref{e:log_approx_4}, and $(b)$ uses $(1-p-\epsy/K)^{-1} \leq (1-p-\epsy)^{-1}$.

Substituting \eqref{e:ItF_counter_itm_f2}, \eqref{e:ItF_counter_itm_s2}, and \eqref{e:ItF_counter_itm_l4} into \eqref{e:ItF_counter_itm}, we have:
\begin{align}
g_t^F(\pi)
& \leq \frac{(K-1)\epsy^2}{K^2}  \left ( \frac{1}{(p + \epsy / K )(1 - p - \epsy / K)} - \frac{1}{2 p(1-p)}  \right) + \tilde{g}(\epsy, K, p)
\label{e:IG_counter_INT}
\\[0.5em]
\tilde{g}(\epsy, K, p) & = \frac{\epsy^3}{3 K^2 (K-1) (1-p)^2} + \frac{\epsy^3}{3p^2} + \frac{(K-1)^2 \epsy^3}{2 K^3 (1 - p - \epsy)^2}
\label{e:IG_g_counter}
\end{align}

To obtain the final bound \eqref{e:IG_counter}, we need to choose $\epsy$ small enough so that:
\begin{subequations}
\begin{align}
\frac{1}{(p + \epsy / K )(1 - p - \epsy / K)} 
& \leq \frac{2}{p(1-p)}
\label{e:first_term_bd_IG}
\\
\frac{\epsy^3}{3 K^2 (K-1) (1-p)^2} + \frac{\epsy^3}{3p^2} + \frac{(K-1)^2 \epsy^3}{2 K^3 (1 - p - \epsy)^2} & \leq  \frac{(K-1)\epsy^2}{K^2} \cdot \frac{1}{p (1-p)}
\label{e:second_term_bd_IG}
\end{align}
\end{subequations}
Applying \eqref{e:epsy_condition_2}-\eqref{e:epsy_condition_4}, we obtain \eqref{e:first_term_bd_IG} . And applying \eqref{e:epsy_condition_6}-\eqref{e:epsy_condition_8}, we obtain \eqref{e:second_term_bd_IG}.

Substituting \eqref{e:first_term_bd_IG} and \eqref{e:second_term_bd_IG} into \eqref{e:IG_g_counter} and \eqref{e:IG_counter_INT} gives us the final bound \eqref{e:IG_counter}.

We next show that \eqref{e:reg_counter} holds.  Using \eqref{e:Regret_Prop2_RVR} of Lemma~\ref{t:IG_Regret_Prop2_RVR}, 
$$
\begin{aligned}
\Delta_t(\pi) 
& = \sum_{a \in \clA} \Pr_t(A_* = a) \E_t \big[R_{t+1, a} | A_* \! = \! a \big] -  \sum_{a \in \clA}  \pi(a) \E_t \big[ R_{t+1,a} \big]
\\
& = p + \epsy -  \sum_{a \in \clA} \pi(a) \cdot \Big( p + \frac{\epsy}{K} \Big)
\\
& = \frac{K-1}{K} \cdot \epsy
\end{aligned}
$$
where we have used \eqref{e:Post_R_counter_1} and \eqref{e:Post_R_marginals_1} along with the fact that $\Pr_t(A_* = a) = 1/K$ for each $a \in \clA$ to obtain the second equality.

The final bound \eqref{e:IR_counter} follows from definition \eqref{e:IR}.

\end{proof}

\newpage

\section{Strengthening the Lower Bound of Section~\ref{s:lower_bound_proof}}
\label{s:lower_bound_proof_rand_pol}

\subsection{Definitions and Goals}

We first introduce some definitions that are useful for extending our results in the main draft.

For any potential $F$ that satisfies $\mathrm{diam}_F(\Delta^{K-1}) < \infty$, and $\alpha \in \Re^+$, we consider the following generalized definition of information ratio $\Gamma_t^F (\pi,\alpha)$ associated with policy $\pi$ at time-step $t$:
\begin{equation}
\Gamma_t^F (\pi,\alpha)
\eqdef \frac{\big( \Delta_t(\pi) - \alpha  \big)^2}{g_t^F(\pi)}
\label{e:IR_alpha}
\end{equation}
The above definition of the information ratio is the same as the one considered in \cite{latsze19}  (see for example, Corollary~4), but written in a different form.
The definition of information ratio in \eqref{e:IR} (and in Section~\ref{s:lower_bound_proof}) is a special case of \eqref{e:IR_alpha}, with $\alpha = 0$. 

For any potential $F$ that satisfies $\mathrm{diam}_F(\Delta^{K-1}) < \infty$, $T \geq 0$, policy $\pi$, $\alpha \in \Re^+$, and $\delta \in [0,1]$, we define
\begin{align}
{\mathrm{Regret}}^{\mathrm{F}}(T, \pi, \alpha,\delta) 
& \eqdef 
\sqrt{\gamma_F(T, \pi, \alpha, \delta) \cdot T \cdot {\mathrm{diam}_F(\Delta^{K-1})}}
+ \left( \alpha + \delta \cdot \mathrm{Reg}_{\max} \right) \cdot T
\label{e:Rprog_def}
\end{align}
where $\gamma_F(T,\pi,\alpha, \delta)$ is any deterministic constant that satisfies, for each $0 \leq t \leq T-1$,
\begin{align}
\Pr \Big( \Gamma_t^F (\alpha,\pi) \leq \gamma_F(T, \pi, \alpha, \delta) \Big) \geq 1  -  \delta
\label{e:IR_bound}
\end{align}
and
\begin{equation}
\mathrm{Reg}_{\max} = \max_{t \geq 0 \,, a \in \clA} \Delta_t(a)
\label{e:reg_max}
\end{equation}

Note that the upper bound $\gamma_F(\pi)$ defined in Section~\ref{s:IR_BR} (above Corollary~\ref{t:Rprog_uniform_bd}) is a special case of $\gamma_F(T,\pi,\alpha, \delta)$, with $\alpha = \delta = 0$, and $T = \infty$. 
And  the right hand side of \eqref{e:Rprog_uniform_bd} in Corollary~\ref{t:Rprog_uniform_bd} is a special case of ${\mathrm{Regret}}^{\mathrm{F}}(T, \pi, \alpha,\delta)$, with $\alpha = \delta = 0$.

The goal of this section is to show that,
\begin{romannum}
\item generalizing the definition of information in \eqref{e:IR} to the one in \eqref{e:IR_alpha}, and
\item replacing a uniform (in time) a.s. upper bound on the information ratio, such as the one used in Corollary~\ref{t:Rprog_uniform_bd}, with a uniform (in time) high probability upper bound as in \eqref{e:IR_bound}
\end{romannum}
both are insufficient to obtain an order $\sqrt{KT}$ bound for \emph{any} policy, using Shannon information ratio analysis. 

This highlights the need for analysis techniques that take into account the temporal nature of information ratio that was introduced in Section~\ref{s:time_var_IR}.

\subsection{An Upper Bound on ${\mathrm{Regret}}$ using ${\mathrm{Regret}}^{\mathrm{F}}$}

As a first step, the following result shows how ${\mathrm{Regret}}^{\mathrm{F}}$ can be used to upper bound the Bayesian regret of any policy $\pi$. The proof uses Lemma~\ref{t:div_bd} that is proved in Section~\ref{s:RprogExt_proof}.

\begin{theorem}
\label{t:True_Rprog}
For any policy $\pi$, $T \geq 1$, $K \geq 2$, $\alpha \in \Re^+$, $\delta \in [0,1]$,
\begin{equation}
{\mathrm{Regret}}(T, \pi) \leq {\mathrm{Regret}}^{\mathrm{F}}(T,\pi,\alpha,\delta)
\label{e:True_Rprog}
\end{equation}
where $F$ is convex, and satisfies $\mathrm{diam}_F(\Delta^{K-1}) < \infty$.
\end{theorem}

\begin{proof}
Recalling the definition of Bayesian regret \eqref{e:regret_bayes},
\begin{align}
\mathrm{Regret}(T, \pi) 
& = \E \left[\sum_{t=0}^{T-1} (R_* - R_{t+1, A_t})\right]
\nonumber
\\
& = \E \left[ \sum_{t=0}^{T-1} \left (R_* - \alpha - R_{t+1, A_t}\right) \right] + \alpha \cdot T
\nonumber
\\
& = \E \left[ \sum_{t=0}^{T-1} \left (\E_t \left[ R_* - \alpha - R_{t+1, A_t} \right] \right) \right]  + \alpha \cdot T
\label{e:regret_ext_bd_proof_0_Rprog}
\end{align}

From the generalized definition of information ratio in \eqref{e:IR_alpha},
\begin{align}
\E \Big[ \Expect_t \left[ R_{*} - \alpha - R_{t+1,A_t} \right] \Big]
& = \E \left[ \sqrt{ \Gamma_t^F (\pi,\alpha) \cdot \Expect_t \big[ D_F \big( P_{t+1}(A_*), P_t(A_*) \big ) \big] } \, \right]
\nonumber
\\[0.5em]
& \leq \sqrt{\gamma_F(T, \pi, \alpha, \delta)}  \cdot \E \left[ \sqrt{ \Expect_t \big[ D_F \big( P_{t+1}(A_*), P_t(A_*) \big ) \big] } \, \right] + \delta \cdot  \mathrm{Reg}_{\max}
\label{e:regret_IR_hp_Rprog}
\end{align}
where $\gamma_F(T, \pi, \alpha, \delta)$ is a high probability upper bound on the information ratio that satisfies \eqref{e:IR_bound}, and we have used the following two inequalities to obtain \eqref{e:regret_IR_hp_Rprog}: For each $0 \leq t \leq T-1$,
\[
\begin{aligned}
\Pr \left(\Gamma_t^F (\pi,\alpha) \leq \gamma_F(T, \pi, \alpha, \delta) \right) & \geq 1 - \delta
\\
\E_t \left[ R_* - \alpha - R_{t+1,A_t} \right] & \leq \mathrm{Reg}_{\max}
\end{aligned}
\]

Using \eqref{e:regret_IR_hp_Rprog} in \eqref{e:regret_ext_bd_proof_0_Rprog}:
\begin{equation}
\begin{aligned}
\mathrm{Regret}(T,\pi) 
& \leq \sqrt{\gamma_F(T, \pi, \alpha, \delta)} \cdot \E \left[ \sum_{t=0}^{T-1} \sqrt{ \Expect_t \big[ D_F \big( P_{t+1}(A_*), P_t(A_*) \big ) \big] } \right] + \left(  \alpha +  \delta \cdot \mathrm{Reg}_{\max} \right) \cdot T
\end{aligned}
\end{equation}

Applying Cauchy Schwarz 
and then using Lemma~\ref{t:div_bd},
\begin{equation*}
\begin{aligned}
\mathrm{Regret}(T, \pi)
& \overset{(a)}{\leq} 
\sqrt{\gamma_F(T, \pi, \alpha, \delta) \cdot T \cdot \E \left[ \sum_{t=0}^{T-1} \Expect_t \big[ D_F \big( P_{t+1}(A_*), P_t(A_*) \big ) \big]  \right] } + \left(  \alpha +  \delta \cdot \mathrm{Reg}_{\max} \right) \cdot T
\\
& \overset{(b)}{\leq} \sqrt{\gamma_F(T, \pi, \alpha, \delta) \cdot T \cdot \E \left[ \sum_{t=0}^{T-1} \Expect_t \big[ F \big( P_{t+1} (A_* ) \big) \big] - F\big(P_{t} (A_*) \big)  \right] } + \left(  \alpha +  \delta \cdot \mathrm{Reg}_{\max} \right) \cdot T
\\
& = \sqrt{\gamma_F(T, \pi, \alpha, \delta) \cdot T \cdot \E \left[ \sum_{t=0}^{T-1} F \big( P_{t+1} (A_* ) \big) - F\big(P_{t} (A_*) \big)  \right] } + \left(  \alpha +  \delta \cdot \mathrm{Reg}_{\max} \right) \cdot T
\\
& \overset{(c)}{\leq} \sqrt{\gamma_F(T, \pi, \alpha, \delta) \cdot T \cdot \mathrm{diam}_F(\Delta^{k-1}) } + \left(  \alpha +  \delta \cdot \mathrm{Reg}_{\max} \right) \cdot T
\\
& = {\mathrm{Regret}}^{\mathrm{F}}(T,\pi,\alpha,\delta)
\end{aligned}
\end{equation*}
where $(a)$ is an application of Cauchy Schwarz, $(b)$ is from Lemma~\ref{t:div_bd}, and $(c)$ follows from the definition of $\mathrm{diam}_F(\cdot)$.

\end{proof}


\subsection{A More General Lower Bound}

Here, we will show the following lower bound for Example~\ref{e:hardproblems}, with $\epsy \in (0, \epsy^*]$, where $\epsy^*$ satisfies \eqref{e:epsy_conditions}.
\begin{theorem}
\label{t:R_prog_lb_alpha_delta_nz}
For all $K \geq 2$, $p \in (0,1)$, and $\epsy \in (0,\epsy^*]$, there exists $T \in \N$, such that, for any policy $\pi$, $\alpha \in \Re^+$, and $\delta \in [0,1]$,
\begin{equation}
{\mathrm{Regret}}^{\mathrm{F}}(T,\pi, \alpha, \delta) \geq \sqrt{{\frac{2}{5}} \cdot p(1-p) \cdot (K-1) \cdot T \cdot \log(K)}
\label{e:R_prog_2_lb_alpha_delta_nz_final}
\end{equation}
\end{theorem}
The following is a direct Corollary to Theorem~\ref{t:R_prog_lb_alpha_delta_nz}, which follows from the fact that $K\geq2$, and $p=1/2$ maximizes the right hand side of \eqref{e:R_prog_2_lb_alpha_delta_nz_final}.
\begin{corollary}
For all $K \geq 2$, there exists $T \geq 0$, $p \in (0,1)$, and $\epsilon^* \in (0,1-p)$, such that, for all $\epsilon \in (0,\epsilon^*]$, and any policy $\pi$,
\[
{\mathrm{Regret}}^{\mathrm{F}}(T,\pi, \alpha, \delta) \geq \sqrt{\frac{1}{20} \cdot K \cdot T \cdot \log(K)}
\]
\end{corollary}

{\bf{Organization of Proof of Theorem~\ref{t:R_prog_lb_alpha_delta_nz}:}}
The proof of Theorem~\ref{t:R_prog_lb_alpha_delta_nz} is split into three parts. First, in Section~\ref{s:proof_lb_alpha_zero}, we show that the lower bound holds for $\alpha = \delta = 0$. Next, in Section~\ref{s:proof_lb_alpha_nz}, we show that the lower bound holds for any $\alpha \geq 0$ and $\delta = 0$. Finally, we show in Section~\ref{s:proof_lb_delta_nz} that the result also holds for all $\alpha \geq 0$ and $\delta \in [0,1]$.

\subsection{Proof of lower bound for $\alpha = \delta = 0$}
\label{s:proof_lb_alpha_zero}

We show the following Proposition in this section.
\begin{proposition}[Lower Bound for $\alpha = \delta = 0$]
\label{t:R_prog_lb_alpha_zero}
For all $K\geq 2$, $p \in (0,1)$, $\epsy \in (0,\epsy^*]$, $T \in \N$, policy $\pi$, and $\alpha = \delta = 0$,
\begin{equation}
{\mathrm{Regret}}^{\mathrm{F}}(T,\pi,\alpha,\delta) \geq \sqrt{{\frac{2}{5}} \cdot p(1-p) \cdot (K - 1) \cdot T \cdot \log(K)}
\label{e:R_prog_lb_alpha_zero_final}
\end{equation}
\end{proposition}

\begin{proof}[\bf Proof of Proposition~\ref{t:R_prog_lb_alpha_zero}]
The proof directly follows from the definition of ${\mathrm{Regret}}^{\mathrm{F}}$ in \eqref{e:Rprog_def} and \eqref{e:IR_counter} of Proposition~\ref{t:IR_counter}: Since we have shown that when $\alpha = 0$ and $t = 0$,
\[
\Gamma_t^F{(\pi, \alpha)} = \Gamma_t^F{(\pi)}  \geq {{\frac{2}{5}} \cdot p(1-p) \cdot (K-1)} \,\,\,\, a.s.,
\]
we have
\[
\gamma_F(T, \pi, \alpha, \delta)   \geq {{\frac{2}{5}} \cdot p(1-p) \cdot (K-1)} 
\]
Substituting $\alpha = \delta = 0$ in \eqref{e:Rprog_def}, and using the fact that ${\mathrm{diam}_F(\Delta^{K-1})} = \log(K)$ when $F$ is negentropy, we obtain \eqref{e:R_prog_lb_alpha_zero_final}.
\end{proof}

\subsection{Proof of Lower Bound for $\delta = 0$ and $\alpha \geq 0$}
\label{s:proof_lb_alpha_nz}

Here we generalize the result of Proposition~\ref{t:R_prog_lb_alpha_zero} for $\alpha \geq 0$.

\begin{proposition}[Lower Bound for $\delta = 0$ and $\alpha \geq 0$]
For all $K \geq 2$, $p \in (0,1)$, $\epsy \in (0,\epsy^*]$, $\alpha \in \Re^+$, $T \in \N$, and $\delta =0$, and any policy $\pi$, 
\begin{equation}
{\mathrm{Regret}}^{\mathrm{F}}(T,\pi,\alpha,\delta) \geq \sqrt{{\frac{2}{5}} \cdot p(1-p) \cdot \frac{\left(\epsy (K-1) - \alpha K \right)^2}{(K-1) \cdot \epsy^2} \cdot T \cdot \log(K)}
+ \alpha T
\label{e:R_prog_lb_alpha_nz}
\end{equation}
Consequently, if $T \in \N$ satisfies
\begin{equation}
T \geq \frac{2}{5} \cdot p (1-p)\cdot \frac{K^2}{(K-1)} \cdot \frac{\log(K)}{\epsy^2}
\label{e:Tcond_2_R_prog_lb_alpha_nz}
\end{equation}
then,
\begin{equation}
\min_{\alpha \in \Re^+} {\mathrm{Regret}}^{\mathrm{F}}(T,\pi,\alpha,\delta) \geq \sqrt{{\frac{2}{5}} \cdot p(1-p) \cdot (K-1) \cdot T \cdot \log(K)}
\label{e:R_prog_2_lb_alpha_nz_final}
\end{equation}
\label{t:R_prog_lb_alpha_nz}
\end{proposition}

The proof of Proposition~\ref{t:R_prog_lb_alpha_nz} relies on the following result.

\begin{lemma}
For all $K \geq 2$, $p \in (0,1)$, $\epsy \in (0, \epsy^*]$, and $\alpha \in \Re^+$, the following holds a.s. at $t = 0$: 
\begin{align}
g_t^F(\pi)
& \leq \frac{(K-1)\epsy^2}{K^2}  \cdot \frac{5}{2 p(1-p)}
\label{e:IG_counter_alpha_nz}
\\[0.5em]
\Delta_t(\pi) - \alpha
& = \frac{\epsy (K-1)}{K} - \alpha
\label{e:reg_counter_alpha_nz}
\end{align}
Consequently, for all $K \geq 2$, $p \in (0,1)$, $\epsy \in (0,\epsy^*]$, $\alpha \in \Re^+$, and any policy $\pi$, at $t = 0$,
\begin{equation}
\Gamma_t^F{(\pi,\alpha)} \geq {\frac{2}{5}} \cdot p(1-p) \cdot \frac{\left(\epsy (K-1) - \alpha K \right)^2}{(K-1) \cdot \epsy^2} \,\,\, a.s..
\label{e:IR_counter_alpha_nz}
\end{equation}
\label{t:IR_counter_alpha_nz}
\end{lemma}

\begin{proof}
The proof of \eqref{e:IG_counter_alpha_nz} is the same as proof of \eqref{e:IG_counter} in Proposition~\ref{t:IR_counter}. The proof of \eqref{e:reg_counter_alpha_nz} is also straightforward from \eqref{e:reg_counter}. The final inequality \eqref{e:IR_counter_alpha_nz} is direct from the definition of generalized information ratio in \eqref{e:IR_alpha}.
\end{proof}

\begin{proof}[\bf Proof of Proposition~\ref{t:R_prog_lb_alpha_nz}]

Recall that $\gamma_F(T, \pi, \alpha, \delta)$ is a deterministic constant that satisfies \eqref{e:IR_bound}. Lemma~\ref{t:IR_counter_alpha_nz} implies, for $\delta =0$,
\[
\gamma_F(T, \pi, \alpha, \delta) \geq {\frac{2}{5}} \cdot p(1-p) \cdot \frac{\left(\epsy (K-1) - \alpha K \right)^2}{(K-1) \cdot \epsy^2}
\]
Using this lower bound in the definition of 
$\mathrm{Regret}^{\mathrm{F}}$ (cf. \eqref{e:Rprog_def}), we obtain \eqref{e:R_prog_lb_alpha_nz}.

The right hand side of \eqref{e:R_prog_lb_alpha_nz} can be simplified as follows:
\begin{align}
& \sqrt{{\frac{2}{5}} \cdot p(1-p) \cdot  \frac{\left(\epsy (K-1) - \alpha K \right)^2}{(K-1) \cdot \epsy^2} \cdot T \cdot \log(K)}
+ \alpha T
\nonumber
\\
& = \left|(K-1) - \frac{\alpha K}{\epsy} \right| \cdot \sqrt{\frac{2 p(1-p)}{5 (K-1)} \cdot T \cdot \log(K)}
+ \alpha T
\nonumber
\\
& = \begin{cases}
\sqrt{\frac{2}{5} \cdot p(1-p) \cdot (K-1) \cdot T \cdot \log(K)}
+ \alpha \left(T - \frac{K}{\epsy} \cdot \sqrt{\frac{2 p(1-p)}{5 (K-1)} \cdot T \cdot \log(K)} \right)
& \quad \text{if} \quad \alpha \leq \frac{\epsy (K-1)}{K}
\\[1em]
- \sqrt{\frac{2}{5} \cdot p(1-p) \cdot (K-1) \cdot T \cdot \log(K)}
+ \alpha \left(T + \frac{K}{\epsy} \cdot \sqrt{\frac{2 p(1-p)}{5 (K-1)} \cdot T \cdot \log(K)} \right) 
& \quad \text{if} \quad \alpha > \frac{\epsy (K-1)}{K}
\end{cases}
\label{e:R_prog_lb_alpha_nz_simpl}
\end{align}
Optimizing the right hand side of \eqref{e:R_prog_lb_alpha_nz_simpl} over $\alpha$, we have, if $T$ satisfies \eqref{e:Tcond_2_R_prog_lb_alpha_nz}:
$$
T \geq \frac{2}{5} \cdot p (1-p)\cdot \frac{K^2}{(K-1)} \cdot \frac{\log(K)}{\epsy^2}
$$
then \eqref{e:R_prog_2_lb_alpha_nz_final} holds: 
$$
\min_{\alpha \in \Re^+} {\mathrm{Regret}}^{\mathrm{F}}(T,\pi,\alpha,\delta) \geq \sqrt{{\frac{2}{5}} \cdot p(1-p) \cdot (K-1) \cdot T \cdot \log(K)}
$$
\end{proof}

\subsection{Proof of Theorem~\ref{t:R_prog_lb_alpha_delta_nz}}
\label{s:proof_lb_delta_nz}

From its definition in \eqref{e:reg_max}, note that $\mathrm{Reg}_{\max} = \epsilon$ for the counter example.
Suppose $\delta \in [0,1]$ satisfies 
\[
\delta > \sqrt{\frac{1}{20} \cdot \frac{K \log(K)}{T}} \cdot \frac{1}{\epsilon}
\]
Then, from the definition \eqref{e:Rprog_def}, we have, for any $\alpha \in \Re^+$,
\[
\begin{aligned}
{\mathrm{Regret}}^{\mathrm{F}}(T,\pi,\alpha,\delta) 
& \geq T \cdot \delta \cdot \epsilon 
\\
& > \sqrt{\frac{1}{20} KT \log(K)}
\end{aligned}
\]
implying that \eqref{e:R_prog_2_lb_alpha_delta_nz_final} holds. In the rest of the proof we will assume 
\begin{equation}
\delta \leq \sqrt{\frac{1}{20} \cdot \frac{K \log(K)}{T}} \cdot \frac{1}{\epsilon}
\label{e:delta_cond}
\end{equation}

Now suppose 
\begin{equation}
T \geq \max \left \{  \frac{K \log(K)}{5 \epsilon^2}\,\, , \quad \frac{2}{5} \cdot p (1-p)\cdot \frac{K^2}{(K-1)} \cdot \frac{\log(K)}{\epsilon^2} \right \}
\label{e:T_large}
\end{equation}
Then, \eqref{e:delta_cond} implies $\delta \leq 0.5$. 

Inequality \eqref{e:IR_counter_alpha_nz} of 
Lemma~\ref{t:IR_counter_alpha_nz} implies that for all $\epsilon \in (0, \epsilon^*]$, and $t = 0$, 
\begin{equation}
\Gamma_t^F{(\pi,\alpha)} \geq {\frac{2}{5}} \cdot p(1-p) \cdot \frac{\left(\epsy (K-1) - \alpha K \right)^2}{(K-1) \cdot \epsy^2} \,\,\, a.s..
\label{e:Gamma_t_bd_alpha}
\end{equation}
Recall that $ \gamma_F(T, \pi, \alpha, \delta)$ is any deterministic constant that satisfies \eqref{e:IR_bound}.
Inequality \eqref{e:Gamma_t_bd_alpha} implies that for large $T$ that satisfies \eqref{e:T_large} (since we have $\delta \leq 0.5$ in this case),
\[
\gamma_F(T, \pi, \alpha, \delta)
\geq {\frac{2}{5}} \cdot p(1-p) \cdot \frac{\left(\epsilon (K-1) - \alpha K \right)^2}{(K-1) \cdot \epsilon^2}
\]
Using the above bound in \eqref{e:Rprog_def},
\begin{equation}
{\mathrm{Regret}}^{\mathrm{F}}(T,\pi,\alpha,\delta) 
\geq \sqrt{{\frac{2}{5}} \cdot p(1-p) \cdot \frac{\left(\epsilon (K-1) - \alpha K \right)^2}{(K-1) \cdot \epsilon^2} \cdot T \cdot \log(K)}
+ \alpha T
\end{equation}
Following along the lines of Proposition~\ref{t:R_prog_lb_alpha_nz} (in particular, see \eqref{e:R_prog_2_lb_alpha_nz_final}), it follows that for all $T$ that satisfies \eqref{e:T_large},
\[
{\mathrm{Regret}}^{\mathrm{F}}(T,\pi,\alpha,\delta) 
\geq \sqrt{{\frac{2}{5}} \cdot p(1-p) \cdot (K-1) \cdot T \cdot \log(K)}
\]
for all $\alpha \in \Re^+$, and $\delta \in [0,1]$.
\qed

\newpage

\section{Proofs of Theorems~\ref{t:Rprog} and~\ref{t:RprogExt}}
\label{s:RprogExt_proof}

Both results will need the following result which is taken from proof of Theorem~3 in \cite{latsze19}.
\begin{lemma}
\label{t:div_bd}
For any convex function $F: \Delta^{K-1} \to \Re \cup \{\infty\}$,
$$\Expect_t \big[ D_F \big( P_{t+1}(A_*), P_t(A_*) \big ) \big] \leq \Expect_t \big[ F \big( P_{t+1} (A_* ) \big) \big] - F\big(P_{t} (A_*) \big)$$
\end{lemma}
\begin{proof}
For each $0 \leq t \leq T-1$, let $\clF_t$ denote the $\sigma$-algebra generated by $H_t$. Note that since $H_0 = \emptyset$, $\clF_0 = \{\emptyset, \Omega\}$.

Note that $\{P_{t+1}(A_*) : 0 \leq t \leq T-1\}$ is a Martingale adapted to $\{\clF_t: 0 \leq t \leq T-1\}$: 
\begin{equation}
\E_t \big[ P_{t+1}(A_*) \big] = P_t(A_*)
\label{e:PtMartingale}
\end{equation}
From the definition of Bregman divergence \eqref{e:BregmanDiv},
\begin{align}
& \E_t \left[ D_F(P_{t+1}(A_*) , P_{t}(A_*)) \right] 
\nonumber
\\
&\hspace{0.4in} = \E_t \left[ \left( F(P_{t+1}(A_*)) \!-\! F(P_t(A_*)) 
- \frac{ F\left ( h P_{t+1}(A_*) \!+\! (1\!-\!h) P_t(A_*) \right) - F\left(P_t(A_*)\right)}{h}
\right) \right]
\nonumber
\\
& \hspace{0.4in} \overset{(a)}{\leq}  \liminf_{h \to 0+}  \left( \E_t \left[ F(P_{t+1}(A_*)) \!-\! F(P_t(A_*)) 
- \frac{ F\left ( h P_{t+1}(A_*) \!+\! (1\!-\!h) P_t(A_*) \right) - F\left(P_t(A_*)\right)}{h}
\right] \right)
\nonumber
\\
& \hspace{0.4in} = \E_t \left[ F(P_{t+1}(A_*)) \right] - F(P_t(A_*))  +   \liminf_{h \to 0+}  
\frac{ F\left(P_t(A_*)\right) - \E_t \left[ F\left ( h P_{t+1}(A_*) + (1\!-\!h) P_t(A_*)  \right) \right] }{h}
\nonumber
\\
& \hspace{0.4in} \overset{(b)}{\leq} \E_t \left[ F(P_{t+1}(A_*)) \right] - F(P_t(A_*))  +   \liminf_{h \to 0+}  
\frac{ F\left(P_t(A_*)\right) - F\left ( \E_t \left[  h P_{t+1}(A_*) + (1\!-\!h) P_t(A_*)  \right]  \right) }{h}
\nonumber
\\
& \hspace{0.4in} \overset{(c)}{=} \E_t \left[ F(P_{t+1}(A_*)) \right] - F(P_t(A_*)) 
\nonumber
\end{align}
where $(a)$ follows from Fatou's lemma, $(b)$ follows from convexity of $F$, and $(c)$ from \eqref{e:PtMartingale}.

\end{proof}

\begin{proof}[\bf Proof of Theorem~\ref{t:Rprog}]

Recalling the definition of Bayesian regret \eqref{e:regret_bayes}, we have
\begin{align}
\mathrm{Regret}(T,\pi) 
& = \E \left[\sum_{t=0}^{T-1} (R_* - R_{t+1, A_t})\right]
\nonumber
\\
& = \E \left[ \sum_{t=0}^{T-1} \left (\E_t \left[ R_* - R_{t+1, A_t} \right] \right) \right]
\nonumber
\\
& \overset{(a)}{=} \E \left[  \sum_{t=0}^{T-1} \sqrt{ \Gamma_t^F (\pi)}  \cdot \sqrt{ \Expect_t \big[ D_F \big( P_{t+1}(A_*), P_t(A_*) \big ) \big] } \, \right]
\nonumber
\\
& \overset{(b)}{\leq}  \sqrt{ \E \left[  \sum_{t=0}^{T-1}  \Gamma_t^F (\pi) \right]}  \cdot \sqrt{ \E \left[  \sum_{t=0}^{T-1}   \Expect_t \big[ D_F \big( P_{t+1}(A_*), P_t(A_*) \big ) \big] \right] } \, 
\nonumber
\\
& \overset{(c)}{\leq}  \sqrt{\barGamma^F_T(\pi)  \cdot T}  \cdot \sqrt{ \E \left[  \sum_{t=0}^{T-1}    \Expect_t \big[ F \big( P_{t+1} (A_* ) \big) \big] - F\big(P_{t} (A_*) \big) \right] } \, 
\nonumber
\\
& \overset{(d)}{=}  \sqrt{\barGamma^F_T(\pi)  \cdot T}  \cdot \sqrt{ \E \left[   F \big( P_{T} (A_* ) \big) - F\big(P_{0} (A_*) \big) \right] } \, 
\nonumber
\\
& \overset{(e)}{\leq} \sqrt{ \barGamma^F_T(\pi) \cdot {\mathrm{diam}_F(\Delta^{K-1})} \cdot T  }
\nonumber
\end{align}
where $(a)$ follows from the definition of information ratio in \eqref{e:IR}, $(b)$ follows from H\"{o}lder's inequality, $(c)$ follows from the definition of $\barGamma^F_T(\pi)$ in \eqref{e:IR_avg}, and from Lemma~\ref{t:div_bd}, $(d)$ follows from the fact that the summation in $(c)$ is telescoping, and finally, $(e)$ follows from the definition of ${\mathrm{diam}_F(\Delta^{K-1})}$ in \eqref{e:diam_F}.

\end{proof}

\begin{proof}[\bf Proof of Theorem~\ref{t:RprogExt}]

Recalling the definition of Bayesian regret \eqref{e:regret_bayes}, we have
\begin{align}
\mathrm{Regret}(T, P_*,\pi) 
& = \E \left[\sum_{t=0}^{T-1} (R_* - R_{t+1, A_t})\right]
\nonumber
\\
& = \E \left[ \sum_{t=0}^{T-1} \left (\E_t \left[ R_* - R_{t+1, A_t} \right] \right) \right]
\label{e:regret_ext_bd_proof_0}
\end{align}

From the definition of information ratio in \eqref{e:IR},
\begin{align}
\E \Big[ \Expect_t \left[ R_{*} - R_{t+1,A_t} \right] \Big]
& = \E \left[ \sqrt{ \Gamma_t^F (\pi) \cdot \Expect_t \big[ D_F \big( P_{t+1}(A_*), P_t(A_*) \big ) \big] } \, \right]
\nonumber
\\[0.5em]
& \leq \sqrt{\gamma_{FX}}  \cdot \E \left[ \sqrt{ \Expect_t \big[ D_F \big( P_{t+1}(A_*), P_t(A_*) \big ) \big] } \, \right] + \delta_t \cdot  \epsy
\label{e:regret_IR_hp_spl}
\end{align}
where we have used the following two inequalities to obtain \eqref{e:regret_IR_hp_spl}: For each $0 \leq t \leq T-1$,
\[
\begin{aligned}
\Pr \left( \Gamma_t^F (\pi)  \leq \gamma_{FX} \right) & \geq 1 - \delta_t
\\[0.75em]
\E_t\left[ R_* - R_{t+1,A_t} \right] & \leq \epsy
\end{aligned}
\]

Using \eqref{e:regret_IR_hp_spl} in \eqref{e:regret_ext_bd_proof_0}:
\begin{equation}
\begin{aligned}
\mathrm{Regret}(T,\pi) 
& \leq \sqrt{\gamma_{FX}} \cdot \E \left[ \sum_{t=0}^{T-1} \sqrt{ \Expect_t \big[ D_F \big( P_{t+1}(A_*), P_t(A_*) \big ) \big] } \right] + \epsy \cdot \sum_{t=0}^{T-1} \delta_t
\end{aligned}
\end{equation}

Applying Cauchy Schwarz 
and then applying Lemma~\ref{t:div_bd},
\begin{equation*}
\begin{aligned}
\mathrm{Regret}(T,\pi)
& \overset{(a)}{\leq} 
\sqrt{\gamma_{FX} \cdot T \cdot \E \left[ \sum_{t=0}^{T-1} \Expect_t \big[ D_F \big( P_{t+1}(A_*), P_t(A_*) \big ) \big]  \right] } + \epsy \cdot \sum_{t=0}^{T-1} \delta_t
\\
& \overset{(b)}{\leq} \sqrt{\gamma_{FX} \cdot T \cdot \E \left[ \sum_{t=0}^{T-1} \Expect_t \big[ F \big( P_{t+1} (A_* ) \big) \big] - F\big(P_{t} (A_*) \big)  \right] } +  \epsy \cdot \sum_{t=0}^{T-1} \delta_t
\\
& = \sqrt{\gamma_{FX} \cdot T \cdot \E \left[ \sum_{t=0}^{T-1} F \big( P_{t+1} (A_* ) \big) - F\big(P_{t} (A_*) \big)  \right] } +  \epsy \cdot \sum_{t=0}^{T-1} \delta_t
\\
& \overset{(c)}{\leq} \sqrt{ \gamma_{FX} \cdot T \cdot \mathrm{diam}_F(\Delta^{k-1}) } +  \epsy \cdot \sum_{t=0}^{T-1} \delta_t
\end{aligned}
\end{equation*}
where $(a)$ is an application of Cauchy Schwarz, $(b)$ follows from Lemma~\ref{t:div_bd}, and $(c)$ follows from the definition of $\mathrm{diam}_F(\cdot)$ in \eqref{e:diam_F}.

\end{proof}

\newpage

\subsection{\large Generalizing Theorem~\ref{t:RprogExt} beyond Example~\ref{ex:counter}}

Theorem~\ref{t:RprogExt} was specific to example~\ref{ex:counter}. Here, we generalize this result to propose a new template for analysis that can be used to upper bound the Bayesian regret of any policy $\pi$. Contrary to the template proposed in Section~\ref{s:lower_bound_proof_rand_pol} (and the one that is commonly used in literature), the template we propose here accounts for the temporal nature of the information ratio.

As in Section~\ref{s:lower_bound_proof_rand_pol}, we will consider the generalized definition of information ratio defined in \eqref{e:IR_alpha}:  For any potential $F$ that satisfies $\mathrm{diam}_F(\Delta^{K-1}) < \infty$, and $\alpha \in \Re^+$, the information ratio $\Gamma_t^F (\pi,\alpha)$ associated with policy $\pi$ at time-step $t$ is
\begin{equation*}
\Gamma_t^F (\pi,\alpha)
\eqdef \frac{\big( \Delta_t(\pi) - \alpha  \big)^2}{g_t^F(\pi)}
\end{equation*}

For any potential $F$ that satisfies $\mathrm{diam}_F(\Delta^{K-1}) < \infty$, $T \geq 0$, policy $\pi$, $\bfmalpha = \{\alpha_t: \alpha_t \in \Re^+\,, 0 \leq t \leq T - 1\}$, and $\bfmdelta = \{\delta_t: \delta_t \in [0,1] \,, 0 \leq t \leq T - 1 \}$, define,
\begin{align}
{\mathrm{Regret}}^{\mathrm{F X}}
(T,\pi,\bfmalpha, \bfmdelta) \eqdef
\sqrt{\gamma_{FX}(T,\pi,\bfmalpha, \bfmdelta) \cdot T \cdot  {\mathrm{diam}_F(\Delta^{K-1})}}
+ 
\sum_{t=0}^{T-1} \left( \alpha_t + \delta_t \cdot \mathrm{Reg}_{\max} \right)
\label{e:Rprog_Ext_hp_def}
\end{align}
where $\gamma_{FX}(T,\pi,\bfmalpha, \bfmdelta)$ is any deterministic constant that satisfies, for each $0 \leq t \leq T-1$,
\begin{align}
\Pr \Big( \Gamma_t^F (\alpha_t,\pi) \leq \gamma_{FX}(T,\pi,\bfmalpha, \bfmdelta) \Big) \geq 1  -  \delta_t
\label{e:IR_Ext_hp}
\end{align}
and $\mathrm{Reg}_{\max}$ is defined in \eqref{e:reg_max}:
\[
\mathrm{Reg}_{\max} = \max_{t \geq 0 \,, a \in \clA} \Delta_t(a)
\]

The proof of Theorem~\ref{t:RprogExtGen} follows along exactly the same lines as the proof of Theorem~\ref{t:Rprog} and is thus omitted.	
\begin{theorem}
\label{t:RprogExtGen}
For any policy $\pi$, $T \in \N$, $K\geq2$, $\bfmalpha = \{\alpha_t: \alpha_t \in \Re^+\,, 0 \leq t \leq T - 1\}$ and $\bfmdelta = \{\delta_t: \delta_t \in [0,1] \,, 0 \leq t \leq T - 1\}$,
\begin{align}
{\mathrm{Regret}}(T, \pi) \leq {\mathrm{Regret}}^{\mathrm{F X}}
(T,\pi,\bfmalpha, \bfmdelta)
\nonumber
\end{align}
where $F$ is convex, and satisfies $\mathrm{diam}_F(\Delta^{K-1}) < \infty$.
\end{theorem}
\qed

{\bf Comments on Theorem~\ref{t:RprogExtGen}}

Note that ${\mathrm{Regret}}^{\mathrm{F X}}$ and $\gamma_{FX}$ in \eqref{e:Rprog_Ext_hp_def} and \eqref{e:IR_Ext_hp} are strict generalizations of ${\mathrm{Regret}}^{\mathrm{F}}$ and $\gamma_F$ defined in \eqref{e:Rprog_def} and \eqref{e:IR_bound} of Section~\ref{s:lower_bound_proof_rand_pol}. That is, by letting $\delta_t \equiv \delta$ and $\alpha_t \equiv \alpha$ for each $0 \leq t \leq T-1$, the definitions coincide. 

Crucially, the analysis introduced in this section accounts for the time-varying nature of the information ratio, since, a uniform high probability upper bound such as the one in \eqref{e:IR_bound} may be too strict for analysis, but a bound such as \eqref{e:IR_Ext_hp} allows for flexibility. Therefore, an algorithm that achieves, for example, an order $\sqrt{KT}$ bound using an application of Theorem~\ref{t:RprogExtGen} need not achieve a similar bound using an application of, for example, Corollary~\ref{t:Rprog_uniform_bd} or Theorem~\ref{t:True_Rprog} that use uniform (over time) upper bounds on the information ratio in the analysis. 

This was precisely the case for Thompson sampling applied to Example~\ref{ex:counter}, for which we proved a lower bound of order $\sqrt{KT\log(K)}$ for ${\mathrm{Regret}}^{\mathrm{F}}$ in Section~\ref{s:lower_bound_proof_rand_pol}, and an upper bound of $\sqrt{KT}$ in Proposition~\ref{t:sqrtKTbd_hp} (that is an application of Theorem~\ref{t:RprogExtGen}) with $F$ the negentropy potential.

\newpage

\section{Proof of  Proposition~\ref{t:sqrtKTbd_hp}}
\label{s:sqrtKTbd_hp_proof}

To prove Proposition~\ref{t:sqrtKTbd_hp}, we only need to show \eqref{e:HprobBd_TS}, which is formalized in the following Proposition~\ref{t:hp_gamma_t}. The final bound \eqref{e:R_prog_hp_ExT_ub} then follows from Theorem~\ref{t:RprogExt}.

\begin{proposition}
For all $t \geq 1$, $K\geq 2$, $p \in (0,1)$, and $\epsy \in (0,1-p)$,
\begin{align}
\Pr \left(\Gamma_t^F (\pi^{\mathrm{TS}}) \leq {8} \right) 
& \geq 1 - \frac{1}{\epsy} \cdot \sqrt{\frac{8 K}{t}}
\label{e:Prob_Gamma_t}
\end{align}
\qed
\label{t:hp_gamma_t}
\end{proposition}

\begin{proof}[\bf Proof of Proposition~\ref{t:sqrtKTbd_hp}]

Letting $\gamma_{FX} = 8$, and $\delta_t = \frac{1}{\epsy} \cdot \sqrt{\frac{8K}{t}}$ for each $t \geq 0$, and $\pi = \pi^{\mathrm{TS}}$, it follows from Theorem~\ref{t:RprogExt} and Proposition~\ref{t:hp_gamma_t} that
\begin{align}
{\mathrm{Regret}}(T, \pi^{\mathrm{TS}}) 
& \leq \sqrt{ \gamma_{FX} \cdot {\mathrm{diam}_F(\Delta^{K-1})} \cdot T } 
+ \epsy \cdot \sum_{t=0}^{T-1} \delta_t
\nonumber
\\
& = \sqrt{ 8 \cdot \log(K) \cdot T }
+ \sum_{t=0}^{T-1} \sqrt{\frac{8K}{t}}
\nonumber
\\
& \leq \sqrt{ 8 \cdot \log(K) \cdot T }
+ 2 \sqrt{ 8 KT}
\nonumber
\\
& \leq 3 \sqrt{ 8 KT}
\nonumber
\end{align}
\end{proof}

The rest of the section is dedicated to the proof of Proposition~\ref{t:hp_gamma_t}.

Denote:
\begin{equation}
\begin{aligned}
q_t \eqdef \max_a \, \pi^{\mathrm{TS}}(a) = \max_a \Pr_t(A_* = a) 
\label{e:posterior_large_T_eq}
\end{aligned}
\end{equation}

We will first show the following result that upper bounds $\Gamma_t^F(\pi^{\mathrm{TS}})$ as a function of $q_t$.
\begin{proposition}
For each $t \geq 0$, the following holds a.s.:
\begin{subequations}
\begin{align}
\Expect_t \left[ R_{*} - R_{t+1,A_t}   \right] 
& \leq \epsy \cdot (1-q_t^2)
\label{e:reg_bd_qt}
\\[0.5em]
g_t^F(\pi^{\mathrm{TS}})
& \geq 2 \cdot \epsy^2 \cdot q_t^2 \cdot (1-q_t)^2
\label{e:IG_bd_qt}
\\[0.5em]
\Gamma_t^F(\pi^{\mathrm{TS}}) & \leq
\frac{(1 + q_t)^2}{2 \cdot q_t^2}
\label{e:IR_bd_qt}
\end{align}
\end{subequations}
Consequently, if $\frac{1}{2} < q_t \leq 1$ for each $t$, then,
\begin{equation}
\Gamma_{t}^F (\pi^{\mathrm{TS}}) \leq 8
\end{equation}
\label{t:IR_bd_qt}
\end{proposition}

\begin{proof}

Applying Lemma~\ref{t:IG_Regret_Prop2_RVR} (for Thompson sampling, we let $\pi(a) = \Pr_t(A_*=a)$ in the Lemma) and Lemma~\ref{t:Post_R_largeT}:
\begin{align}
\Expect_t \big[ R_{*} - R_{t+1,A_t}   \big]
& \overset{(a)}{=} \sum_{a \in \clA} \Pr_t(A_*\! = \!a) \left(\E_t \big[R_{t+1, a} | A_* \! = \! a \big] - \E_t \big[ R_{t+1,a} \big] \right)
\nonumber
\\
& \overset{(b)}{=} \sum_{a \in \clA} \Pr_t(A_*\! = \!a) \left(\epsy - \epsy \cdot \Pr_t(A_*\! = \!a)  \right)
\nonumber
\\
& = \epsy \cdot \left(1  - \sum_{a \in \clA} \Pr_t^2 (A_*\! = \!a)  \right)
\nonumber
\\
& \leq \epsy \cdot \left(1  - \left[ \max_a \, \Pr_t (A_*\! = \!a) \right]^2 \right)
\nonumber
\\
& = \epsy  \cdot \left(1  -  q_t^2 \right) 
\nonumber
\end{align}
where $(a)$ follows from Lemma~\ref{t:IG_Regret_Prop2_RVR} and $(b)$ follows from Lemma~\ref{t:Post_R_largeT}. We have shown \eqref{e:reg_bd_qt}.

To prove \eqref{e:IG_bd_qt}, we will use the following inequality (which is a consequence of Pinsker's) from \cite{rusvan14} (see Fact~9 and the inequality that follows immediately below on page 15):
\begin{equation}
\Expect_t \big[ R_{t+1,a} | A_* = a_* \big] - \Expect_t \big[ R_{t+1,a} \big] \leq \sqrt{\frac{1}{2} D_{{\mathrm{KL}}} \Big(P_t \big( R_{t+1, a} | A_* \! = \! a_* \big) || P_t \big( R_{t+1,a} \big) \Big)}
\label{e:fact9}
\end{equation}
Combing \eqref{e:IG_Prop2_RVR} of Lemma~\ref{t:IG_Regret_Prop2_RVR} with \eqref{e:fact9}:
\begin{align}
g_t^F(\pi^{\mathrm{TS}})
& = \sum_{a_*, a \in \clA} \Pr_t(A_*\! = \!a) \Pr_t(A_*\! = \!a_*) \Big[ D_{{\mathrm{KL}}} \Big(P_t \big( R_{t+1, a} | A_* \! = \! a_* \big) || P_t \big( R_{t+1,a} \big) \Big) \Big]
\nonumber
\\
& \geq 2 \cdot \sum_{a_*, a \in \clA} \Pr_t(A_*\! = \!a) \Pr_t(A_*\! = \!a_*) \cdot \left( \Expect_t \big[ R_{t+1,a} | A_* = a_* \big] - \Expect_t \big[ R_{t+1,a} \big]
\right)^2
\nonumber
\\
& \overset{(a)}{\geq} 2 \cdot \epsy^2 \cdot \left(  \max_a \Pr_t (A_*\! = a) \right)^2 \cdot \left( 1 - \max_a \Pr_t (A_*\! = a)
\right)^2
\nonumber
\\
& = 2 \cdot \epsy^2 \cdot q_t^2 \cdot \left( 1 - q_t \right)^2 
\nonumber
\end{align}
Where $(a)$ follows from Lemma~\ref{t:Post_R_largeT}: we have ignored all terms in the summation except the one corresponding to $a = a_*$, and $a$ being the maximizer of $\Pr_t(A_* = a)$. We have now shown \eqref{e:IG_bd_qt}.

The final inequality \eqref{e:IR_bd_qt} follows from the definition \eqref{e:IR}.

\end{proof}

Next, in the following proposition we show that for large $t$, the posterior $\Pr_t(A_* = \cdot)$ is concentrated.
\begin{proposition}
For all $t \geq 1$, $K\geq 2$, $p \in (0,1)$, and $\epsy \in (0,1-p)$,
\begin{align}
\Pr \left( q_t > \frac{1}{2} \right) 
& \geq 1 - \frac{1}{\epsy} \cdot \sqrt{\frac{8 K}{t}}
\label{e:Pr_t_lb}
\end{align}
\qed
\label{t:hp_ub_q_t}
\end{proposition}

\subsection*{Proof of Proposition~\ref{t:hp_ub_q_t}}

Recall the definition of $\Delta_t$: for each $a \in \clA$, and $t \geq 0$,
$$
\Delta_t(a) = \E_t \left[ R_* - R_{t+1, a} \right]
$$
The following result is a special case of Proposition~8 of \cite{rusvan18l}:
\begin{lemma}[Proposition~8 of \cite{rusvan18l}]
For each $T\geq 1$, suppose the actions $\{A_t: 0 \leq t \leq T-1\}$ are selected according to $\Pr_t(A_t = a) = \Pr_t(A_* = a)$, then, 
\[
\Expect \left[ \min_a \Delta_T(a) \right] \leq \frac{{\mathrm{Regret}}(T,\pi^{\mathrm{TS}})}{T}
\]
\label{t:best_arm_russo_van_roy_18}
\end{lemma}
First, note that the quantity on the left hand side depends on the particular policy (in this case, Thompson sampling), since the posterior that affects $\Delta_T(a)$ is a function of the past actions that are chosen according to the policy. Lemma~\ref{t:best_arm_russo_van_roy_18} says that, if we can bound the worst case regret (over all priors) for Thompson sampling, we can bound $\Expect \left[ \min_a \Delta_T(a) \right]$. And since we know from \cite{latsze19} that for any distribution $P_*$ on $\clP$,
\begin{equation*}
{\mathrm{Regret}}(T, \pi^{\mathrm{TS}}) \leq \sqrt{2 K T}
\end{equation*}
we have:
\begin{equation}
\Expect \left[ \min_a \Delta_T(a) \right] \leq \sqrt{\frac{2K}{T}}
\label{e:best_arm_bd_TS}
\end{equation}

Lemma~\ref{t:max_a_Pt} just follows from definitions.
\begin{lemma} For each $t \geq 0$,
\begin{equation}
\min_a \Delta_t(a) = \epsy \cdot \left(1 - \max_a \Pr_t(A_* = a) \right)
\end{equation}
\label{t:max_a_Pt}
\end{lemma}
\begin{proof}
We just need to observe that
\[
\begin{aligned}
\Delta_t(a) 
& = \E_t \left[ R_* - R_{t+1, a} \right]
\\
& = p + \epsy - p - \epsy \cdot \Pr_t(A_* = a)
\\
& = \epsy \cdot \left(1 - \Pr_t(A_* = a) \right)
\end{aligned}
\]
where the second equality follows from Lemma~\ref{t:Post_R_largeT}.
\end{proof}

We are now ready to give the proof of Proposition~\ref{t:hp_ub_q_t}.
\begin{proof}[\bf Proof of Proposition~\ref{t:hp_ub_q_t}]
Note that:
\begin{equation}
\begin{aligned}
\min_a \Pr_t(A_* \neq a) 
& = 1 - \max_a \Pr_t(A_* = a)
\\
& = \frac{1}{\epsy} \left[ \min_a  \Delta_t(a) \right]
\end{aligned}
\end{equation}
where the second equality follows from Lemma~\ref{t:max_a_Pt}.

Taking expectations on both sides, and then applying \eqref{e:best_arm_bd_TS},
\begin{align}
\Expect \left[ \min_a \Pr_t(A_* \neq a) \right] 
& = \frac{1}{\epsy} \cdot \Expect \left[ \min_a \Delta_t(a) \right]
\nonumber
\\
& \leq \frac{1}{\epsy} \cdot \sqrt{\frac{2K}{t}}
\end{align}
Applying Markov's inequality, we have:
\begin{align}
\Pr \left( \min_a \Pr_t(A_* \neq a) \geq \frac{1}{2} \right) \leq \frac{1}{\epsy} \sqrt{\frac{8K}{t}}
\end{align}
This implies the required bound \eqref{e:Pr_t_lb}:
\begin{align}
\Pr \left( \min_a \Pr_t(A_* \neq a) \geq \frac{1}{2} \right) 
& \leq \frac{1}{\epsy} \sqrt{\frac{8K}{t}}
\nonumber
\\
\implies \Pr \left( \max_a \Pr_t(A_* = a) \leq \frac{1}{2} \right) 
& \leq \frac{1}{\epsy} \sqrt{\frac{8K}{t}}
\nonumber
\\
\implies \Pr \left( \max_a \Pr_t(A_* = a) > \frac{1}{2} \right) 
& \geq 1 - \frac{1}{\epsy} \sqrt{\frac{8K}{t}}
\end{align}

\end{proof}

\begin{proof}[\bf Proof of Proposition~\ref{t:hp_gamma_t}]
It follows directly by combining Proposition~\ref{t:IR_bd_qt} and Proposition~\ref{t:hp_ub_q_t}.
\end{proof}

\newpage

\section{Details of Experimental Results and Additional Experimental Results}
\label{s:exp_details}

\subsection{Implementation Details for Numerical Results in Section~\ref{s:sim_counter}}

Here we give details on implementation of Thompson sampling, Shannon-IDS and Tsallis-IDS applied to Example~\ref{ex:counter}.

For this example, for each $ t \geq 0$, given $H_t = \{A_s, R_{s+1, A_s}:s = 0,1,2,\ldots,t-1\}$, we can compute the posterior distribution on the optimal action:
\begin{equation}
\Pr_t(A_* = a) = \frac{(1+\epsilon/p)^{s_a(t)} (1-\epsilon/(1-p))^{f_a(t)} }{\sum_{a' \in \clA} (1+\epsilon/p)^{s_{a'}(t)} (1-\epsilon/(1-p))^{f_{a'}(t)}} 
\label{e:Pt_counter}
\end{equation}
where,
\[
\begin{aligned}
s_a(t) & = \sum_{s=1}^{t-1} R_{s+1,A_s} \cdot \ind\{A_s = a\} 
\\
f_a(t) & = \sum_{s=1}^{t-1} (1-R_{s+1,A_s}) \cdot \ind\{A_s = a\}
\end{aligned}
\]
are the total number of $1$'s and $0$'s observed from arm $a$ at time $t-1$.
From \eqref{e:Pt_counter}, we can implement each of the algorithms as described below.
\subsubsection{Thompson Sampling}

At each iteration $t\geq0$, Thompson sampling simply chooses action $a$ with probability $\pi^{\mathrm{TS}}(a)$, where $
\pi^{\mathrm{TS}}(a) = \Pr_t(A_* = a)
$.

\subsubsection{Shannon-IDS}

It follows from Lemma~\ref{t:IG_Regret_Prop2_RVR} that at each time-step, the numerator and denominator of the information ratio $\Gamma_t^F(\pi)$ defined in \eqref{e:IR} can be computed using:
\begin{align}
g_t^F(\pi)
& = \sum_{a_*, a \in \clA} \pi(a) \Pr_t(A_* = a_*) \Big[ D_{{\mathrm{KL}}} \Big(P_t \big( R_{t+1, a} | A_* =  a_* \big) || P_t \big( R_{t+1,a} \big) \Big) \Big]
\label{e:IG_Prop2_RVR_counter}
\\
\Delta_t(\pi)
& = \sum_{a \in \clA} \Pr_t(A_* = a) \E_t \big[R_{t+1, a} | A_* \! = \! a \big] -  \sum_{a \in \clA}  \pi(a) \E_t \big[ R_{t+1,a} \big]
\label{e:Regret_Prop2_RVR_counter}
\end{align}

In \eqref{e:IG_Prop2_RVR_counter}, the KL divergence has the following closed form (see \eqref{e:KL_counter}, and the proof of Proposition~\ref{t:IR_counter} for the derivation):
\begin{align}
& D_{{\mathrm{KL}}} \Big(P_t \big( R_{t+1, a} | A_* \! = \! a_* \big) || P_t \big( R_{t+1,a} \big) \Big)
\nonumber
\\
& \hspace{0.8in} \! = \!
\begin{cases}
\displaystyle (1 \! - \! p \! - \! \epsy) \log \left(\frac{1 - p - \epsy}{1 - p - \epsy \cdot \Pr_t(A_* = a)} \right) + (p \! + \! \epsy) \log \left( \frac{p + \epsy}{p + \epsy \cdot \Pr_t(A_* = a)} \right) & \,\,\, \text{If} \,\, a = a_*
\\[1.4em]
\displaystyle (1 \! - p) \log \left(\frac{1 - p}{1 - p - \epsy \cdot \Pr_t(A_* = a)} \right) + p \log \left(\frac{p}{p + \epsy \cdot \Pr_t(A_* = a)}\right) & \,\,\, \text{If} \,\, a \neq a_*
\end{cases}
\label{e:KL_counter_counter}
\end{align}

Similarly, the right hand side of \eqref{e:Regret_Prop2_RVR_counter} can be evaluated for the counter example as (once again, see proof of Proposition~\ref{t:IR_counter} for the derivation):
\[
\Delta_t(\pi) = \epsy -  \sum_{a \in \clA} \pi(a) \cdot \epsy \cdot \Pr_t(A_* = a)
\] 

From the above calculations, at each iteration $t \geq 0$, the Shannon-IDS agent chooses action $a$ with probability $\pi^{\mathrm{NDS}}$:
\begin{equation}
\pi^{\mathrm{NDS}} \in \argmin_{\pi} \,\,  \Gamma_t^F(\pi) \equiv  \argmin_{\pi} \,\,  \frac{ \big[ \Delta_t(\pi)  \big] ^2}{g_t^F(\pi)}
\label{e:Shannon_IDS_Ex}
\end{equation}
In our implementation, we use the fact that it is sufficient to search over all two-action support policies to solve \eqref{e:Shannon_IDS_Ex} (see \cite{rusvan18l}, in particular Algorithm~3).

\subsubsection{Tsallis-IDS}

The Tsallis information gain can be computed using the following expressions: $g_t^F(\pi) = \sum_a \pi(a) g_t^F(a)$, where
\begin{align}
& g_t^F(a)
= \sum_{a_* \in \clA} \sqrt{\Pr_t(A_* = a_*)} \Bigg( \left(\sqrt{\Pr_t \big( R_{t+1, a} = 0  \big)} -  \sqrt{\Pr_t \big( R_{t+1, a} = 0 | A_* =  a_* \big) } \right)^2 
\nonumber
\\
& \qquad  \qquad \qquad \qquad \qquad \qquad \qquad \qquad \qquad \qquad + \left(\sqrt{\Pr_t \big( R_{t+1, a} = 1  \big)} -  \sqrt{\Pr_t \big( R_{t+1, a} = 1 | A_* =  a_* \big) } \right)^2 \Bigg)
\label{e:IG_Tsallis_counter}
\\
&\hspace{-0.25in} =
\begin{cases}
\displaystyle \sum_{a_* \in \clA} \sqrt{\Pr_t(A_* = a_*)} \Big( \left(\sqrt{1 \!-\! p \!-\! \epsy \Pr_t(A_* = a)} -  \sqrt{1\!-\!p\!-\!\epsy} \right)^2  + \left(\sqrt{p \!+\! \epsy \Pr_t(A_* = a)} -  \sqrt{p \!+\! \epsy} \right)^2 \Big) &  a = a_*
\label{e:IG_Tsallis_counter_simplified}
\\
\displaystyle \sum_{a_* \in \clA} \sqrt{\Pr_t(A_* = a_*)} \Big( \left(\sqrt{1 \!-\! p \!-\! \epsy \Pr_t(A_* = a)} -  \sqrt{1\!-\!p} \right)^2  + \left(\sqrt{p \!+\! \epsy \Pr_t(A_* = a)} -  \sqrt{p} \right)^2 \Big) &  a \neq a_*
\end{cases}
\end{align}
The derivation of \eqref{e:IG_Tsallis_counter} follows from the definition of $g_t^F$, and can be found in \cite{latsze19} (see proof of Theorem~7 in Appendix B). The right hand side of \eqref{e:IG_Tsallis_counter} can be simplified to \eqref{e:IG_Tsallis_counter_simplified} using Lemma~\ref{t:Post_R_largeT}. At each iteration $t \geq 0$, computing $g_t^F(\pi)$ from \eqref{e:IG_Tsallis_counter_simplified} is straightforward using \eqref{e:Pt_counter}.

From the above calculations, at each iteration $t \geq 0$, the Tsallis-IDS agent chooses action $a$ with probability $\pi^{\mathrm{TDS}}$:
\begin{equation}
\pi^{\mathrm{TDS}} \in \argmin_{\pi} \,\,  \frac{ \big[ \Delta_t(\pi)  \big] ^2}{g_t^F(\pi)}
\label{e:Tsallis_IDS_Ex}
\end{equation}

Once again, it is sufficient to search over all two-action support policies to solve \eqref{e:Tsallis_IDS_Ex}.

\subsection{Implementation Details for Beta Bernoulli Bandits in Section~\ref{s:sim_beta_ber}}

Here we give details of the implementation for the results in Section~\ref{s:sim_beta_ber}. 

In the Beta-Bernoulli setting, at time $0$, the mean $\theta_a$ of each arm $a \in \clA$ is assumed to be independent and beta-distributed with prior parameters $(\beta_{0,a}^1, \beta_{0,a}^2)$. In our experiments, we let $(\beta_{0,a}^1, \beta_{0,a}^2) = (1,1)$. 

For $t\geq 0$, after taking action $A_t$ and observing $R_{t+1,A_t} \in \{0,1\}$, the posterior parameters $(\beta_{t+1,a}^1, \beta_{t+1,a}^2)$ can be computed using
\[
\begin{aligned}
\beta_{t+1,a}^1 & = \beta_{t,a}^1 + R_{t+1,A_t} \cdot \ind_{A_t = a}
\\
\beta_{t+1,a}^2 & = \beta_{t,a}^2 + (1 - R_{t+1,A_t}) \cdot  \ind_{A_t = a}
\end{aligned}
\] 

\subsubsection{Thompson Sampling}

At time-step $t$, the Thompson sampling agent samples $\theta_a' \sim (\beta_{t,a}^1, \beta_{t,a}^2)$ and chooses action $A_t = a_*$ where:
$$
a_* = \argmax_a \,\, \theta_a'.
$$

\subsubsection{Shannon-IDS}

At each time-step, the Shannon-IDS agent computes the information ratio according to the definition in \eqref{e:IR}. For the Beta-Bernoulli problem, Algorithm~2 of \cite{rusvan18l} (on page 12) can be used to compute the numerator and denominator of the information ratio. The algorithm takes as input the current beta parameters $(\beta_{t,a}^1, \beta_{t,a}^2)$.

As discussed in \cite{rusvan18l}, the algorithm can not readily be implemented on a computer because several steps of the algorithm involves computing integrals of continuous functions. In our implementation, we approximate the integrals using summations via discretization. Once the information ratio is computed, Algorithm~3 of \cite{rusvan18l} can be used to obtain the policy that minimizes the information ratio.

\subsubsection{Tsallis-IDS}

The Tsallis-IDS algorithm follows along the same lines as the Shannon-IDS algorithm, except for a modification of the information gain computation step. Specifically, to compute the information ratio, we use Algorithm~2 of \cite{rusvan18l}, by replacing line 10 of the algorithm with:
\[
\vec{g}_a \leftarrow \sum_{a'} \sqrt{p^*(a')} \cdot \left( \left(\sqrt{\frac{\beta_a^1}{\beta_a^1 + \beta_a^2}} - \sqrt{M_{a|a'}} \right)^2 + \left(\sqrt{\frac{\beta_a^2}{\beta_a^1 + \beta_a^2}} - \sqrt{1 - M_{a|a'}} \right)^2 \right)
\]
The rest of the steps are identical to Shannon-IDS.

\subsection{Additional Experimental Results for Example~\ref{ex:counter}}

For a given potential $F$, and a policy $\pi$, define 
\begin{equation}
\widehat{\Gamma}^F_t (\pi)
\eqdef \frac{\Big( \Expect \big[ \Delta_t(\pi) \big] \Big)^2}{\E[g_t^F(\pi)]}
\label{e:ExpIR}
\end{equation}
Contrary to $\Gamma_t^F$ defined in \eqref{e:IR}, $\widehat{\Gamma}^F_t$ is \emph{not} a random variable, since we are taking expectation over all possible histories in both the numerator and denominator. By a simple modification of the proof of Theorem~\ref{t:Rprog}, it is not difficult to show that (see for example \cite{rusvan18s,donvan18} that consider the special case of Shannon information ratio), for any policy $\pi$,
\begin{equation}
\begin{aligned}
{\mathrm{Regret}}(T,\pi) \! & \leq \! \sqrt{ \widehat{\barGamma}_T(\pi)  \! \cdot \! T \! \cdot \! {\mathrm{diam}_F(\!\Delta^{K-1}\!)}}
\\
\widehat{\barGamma}_T(\pi) & = \frac{1}{T}\sum_{t=0}^{T-1} \widehat{\Gamma}^F_t (\pi) 
\label{e:RegExpIR}
\end{aligned}
\end{equation}

\begin{figure}[ht]
\centering
\includegraphics[trim={2.5cm 6.1cm 2.5cm 6.77cm}, width=0.8\textwidth]{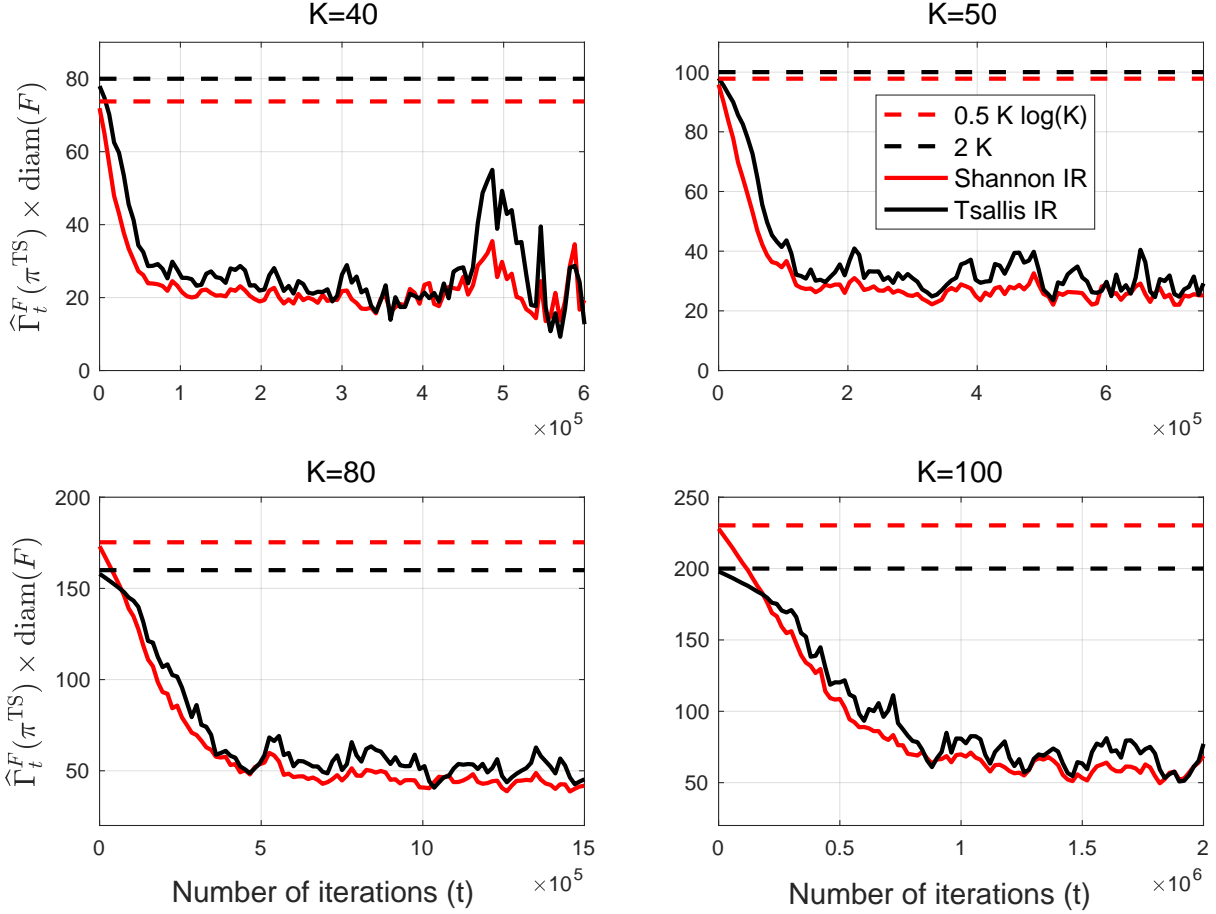}
\vspace{-0.1in}
\caption{Estimate of $\widehat{\Gamma}^F_t$ plotted as a function of time for Thompson sampling applied to the counter example of Section~\ref{s:counter}.}
\label{fig:counter_nrv_info_ratio_TS}
\vspace{-0.1in}
\end{figure}

\begin{figure}[ht]
\centering
\begin{center}
\vspace{-0.1in}
\hbox{\includegraphics[trim={0cm 6.1cm 2.5cm 6cm}, width=0.8\textwidth]{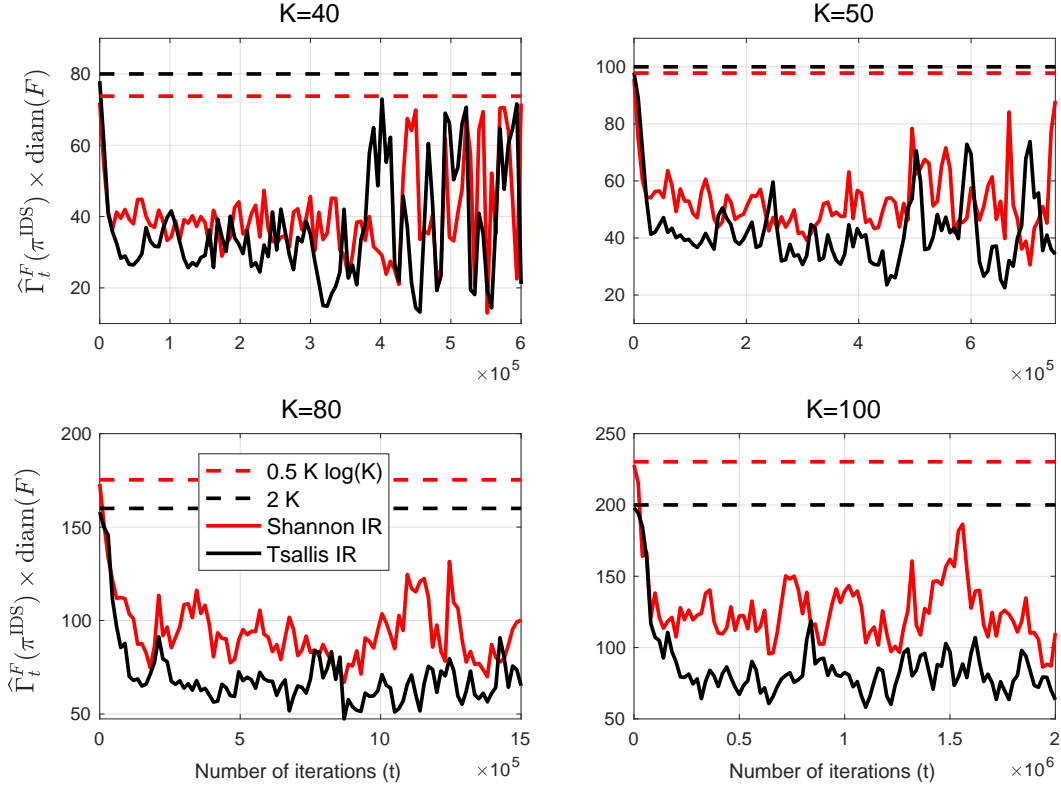}
}
\vspace{-0.1in}
\caption{Estimate of $\widehat{\Gamma}^F_t(\pi^{\mathrm{IDS}})$ plotted as a function of time for Shannon-IDS and Tsallis-IDS applied to the counter example~\ref{ex:counter}.}
\label{fig:counter_nrv_info_ratio_IDS}
\vspace{-0.1in}
\end{center}
\end{figure}

In Figure~\ref{fig:counter_nrv_info_ratio_TS} we plot the estimate of the scaled information ratio $\widehat{\Gamma}^F_t (\pi^{\mathrm{TS}}) \times \mathrm{diam}_F(\Delta^{K-1})$ as a function $t$, for each of the two potential functions: negentropy and $1/2$-Tsallis entropy. The exact expectations in both the numerator and denominator of $\widehat{\Gamma}^F_t (\pi)$ in \eqref{e:ExpIR} was replaced by the empirical averages obtained using the $N$ sample paths. 

It is interesting to see that as $t$ gets large, the information ratio corresponding to both potentials quickly converge, even though they have a noticeable difference at $t=0$. This is especially true when $K$ is large, in which case it is known that the initial difference between the scaled information ratios is large. The dashed lines indicate the worst case bounds on $\widehat{\Gamma}^F_t (\pi) \times \mathrm{diam}_F(\Delta^{K-1})$ for the two potentials.

In Figure~\ref{fig:counter_nrv_info_ratio_IDS} we plot estimate of $\widehat{\Gamma}^F_t (\pi^{\mathrm{IDS}}) \times \mathrm{diam}_F(\Delta^{K-1})$ for the two IDS algorithms. Different from Thompson sampling, we notice that the information ratio decreases more drastically, and after reaching a certain threshold, it seems to stabilize. For $K=40$, we notice that for $t \geq 4 \times 10^5$, there's a lot of chattering of the information ratios. We conjecture that this is the region where the algorithms have identified the optimal arm, and the information gain and the instantaneous regret, both are near zero.

\subsection{Additional Experimental Results for Beta Bernoulli}

\begin{figure}[htbp]
\centering
\begin{center}
\includegraphics[trim={2.5cm 9cm 2.5cm 9.4cm}, width=0.9\textwidth]{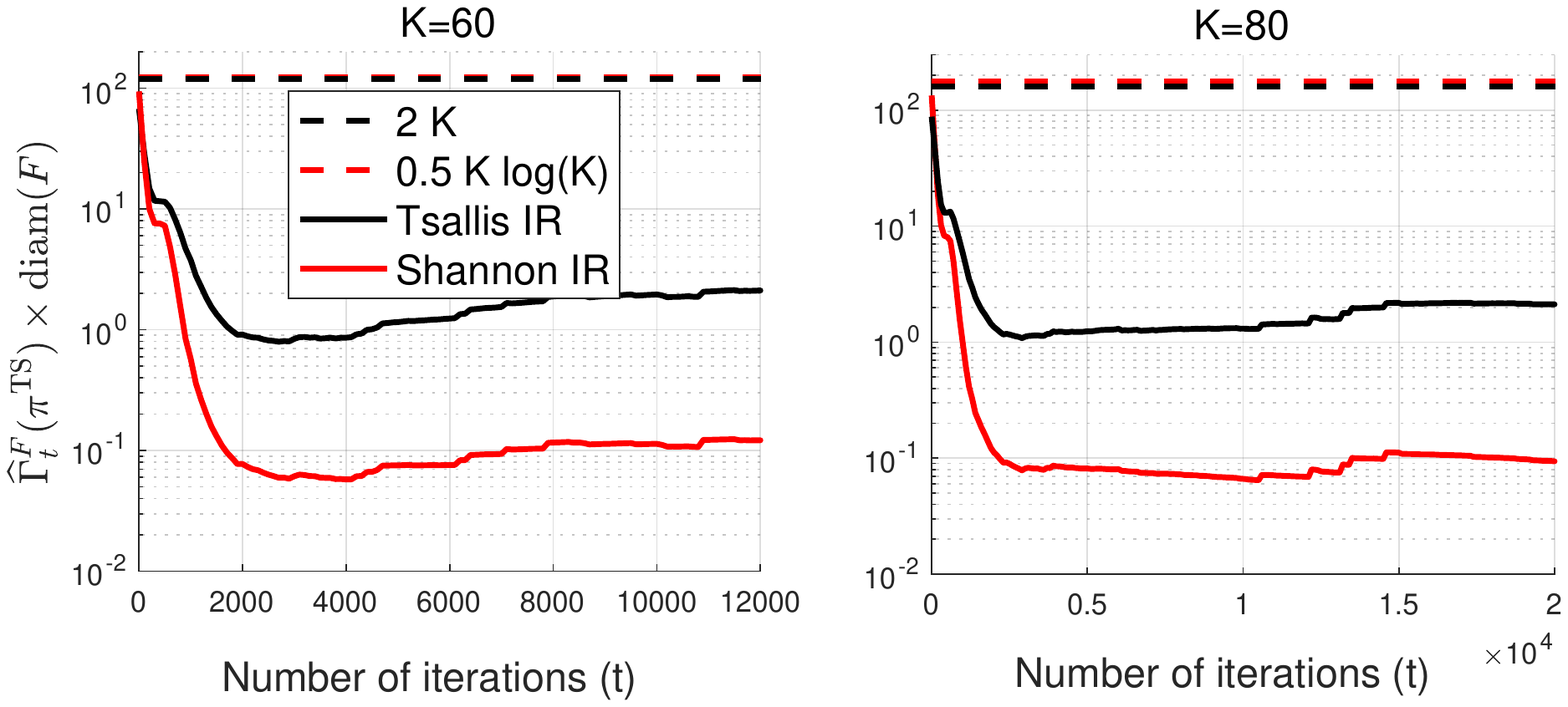}
\caption{Estimate of $\widehat{\Gamma}_t^F(\pi^{\mathrm{TS}})$ as a function of time, for $K=60$ and $K=80$ Beta-Bernoulli.}
\label{fig:beta_bernoulli_non_rv_info_ratio_TS}
\end{center}
\end{figure}

\begin{figure}[htbp]
\centering
\begin{center}
\vspace{-0.1in}
{\includegraphics[trim={2.5cm 9cm 2.5cm 9.4cm}, width=0.9\textwidth]{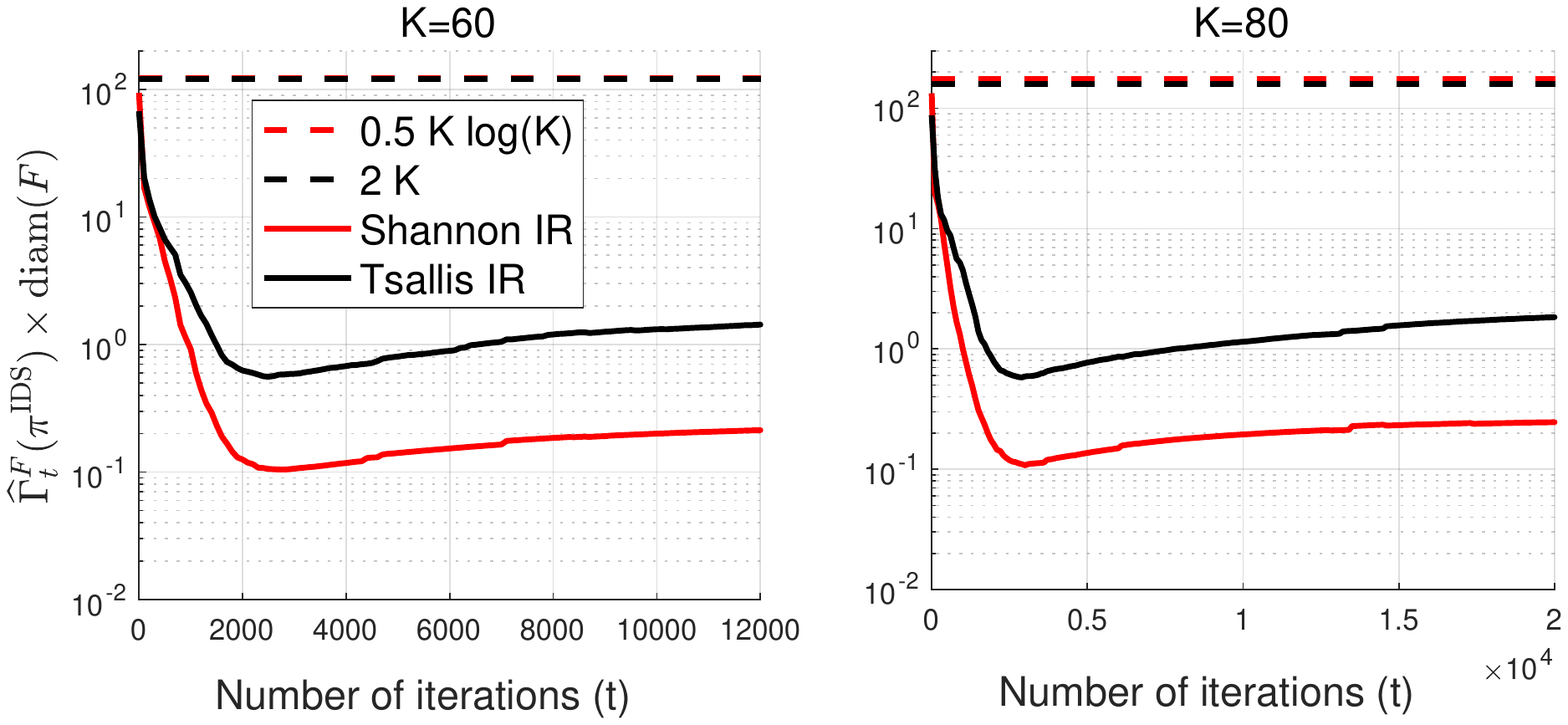}
}
\vspace{-0.1in}
\caption{Empirical average and $2 \sigma$ confidence intervals of $\widehat{\Gamma}_t^F(\pi^{\mathrm{IDS}})$ as a function of time, for $K=60$ and $K=80$ Beta-Bernoulli.}
\label{fig:beta_bernoulli_non_rv_info_ratio_IDS}
\end{center}
\end{figure}

In Figures~\ref{fig:beta_bernoulli_non_rv_info_ratio_TS} and~\ref{fig:beta_bernoulli_non_rv_info_ratio_IDS} we plot scaled $\widehat{\Gamma}^F_t$ for Thompson sampling and IDS algorithms as  a function of time. We observe that the information ratios decrease much more quickly in this experiment, compared to the counter example of Section~\ref{s:counter}. More importantly, we notice that the scaled information ratio for negentropy is consistently lower than the scaled information ratio for the $1/2$-Tsallis entropy, despite the worst case bound being larger for both $K=60$ and $K=80$.

\end{document}